%% file: main.tex
\documentclass{article}

\usepackage{amsthm}
\usepackage{mathtools}
\usepackage{amsmath}
\usepackage{bbm}
\usepackage{amsfonts}
\usepackage{amssymb}
\usepackage{xpatch}
\usepackage{enumitem}
\usepackage{graphicx}
\usepackage{thmtools,thm-restate}
\usepackage{pifont}% http://ctan.org/pkg/pifont % Added
\newcommand{\cmark} {\ding{51}}% % Added
\newcommand{\xmark}{\ding{55}}% % Added
\usepackage{multirow}
\usepackage{wrapfig}
\usepackage{tikz}
\tikzstyle{vertex}=[circle, draw, inner sep=0pt, minimum size=6pt]
\usetikzlibrary{arrows}

\usepackage{xcolor}
\usepackage{colortbl} % Added
\usepackage{tabularx} % Added
\usepackage{makecell}  % Added
\usepackage{tabstackengine} % Added
\usepackage{adjustbox} % Added
\usepackage{footnote}
\usepackage{caption}
\makesavenoteenv{tabular}
\usepackage{titletoc} %Added

\definecolor{mygreen}{RGB}{6, 255, 87}
\usetikzlibrary{fit,calc}
\newcommand*{\tikzmk}[1]{\tikz[remember picture,overlay,] \node (#1) {};\ignorespaces}
\newcommand{\boxit}[1]{\tikz[remember picture,overlay]{\node[yshift=2pt,xshift=-68pt,fill=#1,opacity=.2,fit={(A)($(B)+(0.97\linewidth,.8\baselineskip)$)}] {};}\ignorespaces}
\newtheorem{theorem}{Theorem}
\newtheorem{lemma}{Lemma}
\newtheorem{remark}{Remark}

\newtheorem{definition}{Definition}
\newtheorem{assumption}{Assumption}
\usepackage{thm-restate} %Added

\input{macros}
\usepackage{thm-restate}
\usepackage{algorithm}
\usepackage{chngcntr} % Added to allow different Theorem numberings
\usepackage[noend]{algorithmic}

\usepackage[round]{natbib}
\usepackage[left=1in,right=1in,top=1in,bottom=1in]{geometry}
\usepackage[colorlinks,allcolors=blue]{hyperref}
\usepackage{cleveref}
\usepackage{soul}

\DeclareMathOperator{\Gap}{Gap}
\DeclareMathOperator{\Var}{Var}

\newcommand\numberthis{\addtocounter{equation}{1}\tag{\theequation}}

\renewcommand{\algorithmicrequire}{\textbf{Input:}}
\renewcommand{\algorithmicensure}{\textbf{Output:}}

\title{\rule{\linewidth}{1.5pt}\\ \vspace{3pt}\textbf{Federated Composite Saddle Point Optimization} \vspace{3pt} \\\rule[8pt]{\linewidth}{1pt}}
\author{
\normalsize
\begin{minipage}{0.45\textwidth}
\centering
\textbf{Site Bai} \\
Department of Computer Science \\
Purdue University \\
\small\url{bai123@purdue.edu}
\end{minipage}
\begin{minipage}{0.45\textwidth}
\centering
\textbf{Brian Bullins}\\ 
Department of Computer Science \\
Purdue University \\
\small\url{bbullins@purdue.edu}
\end{minipage}
}
\date{}

\begin{document}
\maketitle

\begin{abstract}
Federated learning (FL) approaches for saddle point problems (SPP) have recently gained in popularity due to the critical role they play in machine learning (ML). Existing works mostly target smooth unconstrained objectives in Euclidean space, whereas ML problems often involve constraints or non-smooth regularization, which results in a need for composite optimization. Addressing these issues, we propose Federated Dual Extrapolation (FeDualEx), an extra-step primal-dual algorithm, which is the first of its kind that encompasses both saddle point optimization and composite objectives under the FL paradigm. Both the convergence analysis and the empirical evaluation demonstrate the effectiveness of FeDualEx in these challenging settings. In addition, even for the sequential version of FeDualEx, we provide rates for the stochastic composite saddle point setting which, to our knowledge, are not found in prior literature.
\end{abstract}
% \vspace{-1.2em}
\section{Introduction}

A notable fraction of machine learning (ML) problems belong to saddle point problems (SPP), including adversarial robustness \citep{madry2018towards, CHEN2023119}, generative adversarial networks (GAN) \citep{NIPS2014Goodfellow}, matrix games \citep{pmlr-v75-abernethy18a}, multi-agent reinforcement learning \citep{wai2018multi}, etc. These applications call for effective distributed saddle point optimization as their scale evolves beyond centralized learning. Federated Learning (FL) \citep{mcmahan17Communication, Konecny2015Federated} is a novel distributed learning paradigm of such where a central server coordinates collaborative learning among clients through rounds of communication. In each round, clients learn a synchronized global model locally without sharing their private data, then send the model to the server for aggregation, usually through averaging \citep{mcmahan17Communication, stich2018local}, to produce a new global model. The cost of communication is known to dominate the FL process \citep{konevcny2016federated}.

While preliminary progress has been made in distributed saddle point optimization \citep{beznosikov2020distributed, hou2021efficient}, we point out that machine learning problems are commonly associated with task-specific constraints or possibly non-smooth regularization, which results in a need for composite optimization (CO). Typical ones include $\ell_1$ norm for sparsity and nuclear norm for low-rankness, which show up in examples spanning from classical LASSO \citep{Tibshirani1996LASSO}, sparse regression \citep{hastie2015statistical} to recent deep learning such as adversarial example generation \citep{Moosavi2016DeepFool, li2022review}, sparse GAN \citep{zhou2020sparse, mahdizadehaghdam2019sparse}, convexified learning \citep{sahiner2022hidden, bai2022dual} and others. Existing distributed methods for SPP fail to cover these composite scenarios as summarized in Table \ref{tab:fed}.

We present the federated learning paradigm for composite saddle point optimization defined in \eqref{eq:obj}. In particular, we propose Federated Dual Extrapolation (FeDualEx) (Algorithm \ref{alg:fed-DualEx}), which builds on Nesterov's dual extrapolation \citep{nesterov2007dual}, a classic extra-step algorithm geared for SPP. It carries out a two-step evaluation of a proximal operator \citep{censor1992proximal} defined by the Bregman Divergence \citep{BREGMAN1967200}, which allows for SPP beyond the Euclidean space. To adapt to composite regularization, FeDualEx also draws inspiration from recent progress in composite convex optimization \citep{yuan2021federated} and adopts the notion of generalized Bregman divergence \citep{flammarion2017stochastic} instead, which merges the regularization into its distance-generating function.
With some novel technical accommodations, we provide the convergence rate for FeDualEx under the homogeneous setting, which is, to the best of our knowledge, the first convergence rate for composite saddle point optimization under the FL paradigm. Furthermore, we  conduct numerical evaluations to verify the effectiveness of FeDualEx on composite SPP.

We also study some other aspects of FeDualEx. First, we notice that \citet{yuan2021federated} identified the ``curse of primal averaging'' in FL from the dichotomy between Federated Mirror Descent (FedMiD) and Federated Dual Averaging (FedDualAvg) \citep{yuan2021federated}, where the specific regularization imposed structure on the client models may no longer hold after primal averaging on the server. Thus, for completeness and comparison, we include the primal twin of FeDualEx based on mirror prox \citep{nemirovski2004prox}, namely ``Federated Mirror Prox (FedMiP)'',  as a baseline in Appendix \ref{appx:FedMiP}. It highlights that FeDualEx naturally inherits the merit of dual aggregation from FedDualAvg. In addition, we analyze FeDualEx for federated composite convex optimization and show that FeDualEx recovers the same convergence rate as FedDualAvg under the convex setting.

Last but not least, by reducing the number of clients to one, we show for the sequential version of FeDualEx that the analysis naturally yields a convergence rate for stochastic composite saddle point optimization which, to our knowledge, is not found in prior literature. Further removing the noise from gradient estimates, FeDualEx still generalizes dual extrapolation to deterministic composite saddle point optimization with a $\mc{O}(\frac{1}{T})$ convergence rate that matches the smooth case and also the pioneering composite mirror prox (CoMP) \citep{he2015mirror} as presented in Table \ref{tab:composite}.

\begin{table*}[t]
\centering
$R$: Communication Rounds. \hfill $K$: Local Steps. \hfill $\beta$: Smoothness. \hfill $B$: Diameter. \hfill $G$: Gradient Bound.
\begin{adjustbox}{max width=\textwidth}
  \begin{tabular}{ccccc}
    \toprule
    \hline
    {Method}  & {Convex} &   {{\makecell{Saddle \\ Point}}} & {\makecell{Composite \\ Objectives}} & {\makecell{Convexity \\ Assumption}} \\
    \toprule
     {\makecell{FedAvg \\ \citep{khaled2020tighter}}} & $\mc{O} \left( \frac{\beta^\frac{1}{3}\sigma^\frac{2}{3}B^\frac{4}{3}}{K^\frac{1}{3}R^\frac{2}{3}} \right)$ & \st{ \ \ \ \ \ \ } & \LARGE \makecell{\xmark} & convex  \\
    \midrule
     {\makecell{FedDualAvg \\ \citep{yuan2021federated}}} & $\mc{O} \left( \frac{\beta^\frac{1}{3}G^\frac{2}{3}B^\frac{2}{3}}{R^\frac{2}{3}} \right)$ & \st{ \ \ \ \ \ \ } & \LARGE \makecell{\cmark} & convex \\
    \bottomrule
    \toprule
    {\makecell{Extra Step Local SGD \\ \citep{beznosikov2020distributed}}} & \st{ \ \ \ \ \ \ } & $\mc{O} \left( B^2 \exp{\{ - \frac{\alpha R}{\beta}\}}{} \right)$ & \LARGE \makecell{\xmark} & \makecell{$\alpha$-strongly  \\  convex-concave} \\
    \midrule
     {\makecell{SCCAFFOLD-S \\ \citep{hou2021efficient}}}& \st{ \ \ \ \ \ \ } & $\mc{O} \left( \frac{\beta^2}{\alpha^2}B^2 \exp{\{ - \frac{\alpha R}{\beta}\}}{} \right)$ & \LARGE \makecell{\xmark} & \makecell{$\alpha$-strongly  \\  convex-concave} \\
     \bottomrule
     \rowcolor{green!20}
     {\tabularCenterstack{c}{\textbf{FeDualEx} \\ \textbf{(Ours)} }} & $\mc{O} \left( \frac{\beta^\frac{1}{3}G^\frac{2}{3}B^\frac{2}{3}}{R^\frac{2}{3}} \right)$ & $\mc{O} \left( \frac{\beta^\frac{1}{2}G^\frac{1}{2}B}{R^\frac{1}{2}} \right)$ & \LARGE \cmark & convex-concave \\
    \hline
    \bottomrule
  \end{tabular}
  \end{adjustbox}
 \caption{We list existing convergence rates on composite convex optimization and smooth saddle point optimization in FL. FedAvg is also included as a reference. Assuming the number of clients is large enough, the dominating term is taken with respect to the rounds of communication $R$. Full complexity is demonstrated in Appendix \ref{appx:review}. We further note that none of the work other than ours covers convex-concave composite SPP. They are included only for completeness.} \vspace{-0pt}
 \label{tab:fed} 
\end{table*} 

\vspace{3pt}
\textbf{\large Our Contributions:}
\begin{itemize} \vspace{-2pt}
    \item We propose FeDualEx for federated learning of SPP with composite possibly non-smooth regularization (Section \ref{sec:FeDualEx-alg}). In support of the proposed algorithm, we provide a convergence rate for FeDualEx under the homogeneous setting (Section \ref{sec:FeDualEx-convergence}). To the best of our knowledge, FeDualEx is the first of its kind that encompasses composite possibly non-smooth regularization for SPP under a federated or distributed paradigm, as shown in Table \ref{tab:fed}. We also present its primal twin FedMiP as a baseline (Appendix \ref{appx:FedMiP}).
    \item Restricting the objective to composite convex functions, FeDualEx achieves the same convergence rate as its counterpart FedDualAvg \citep{yuan2021federated} in federated composite convex optimization (Section \ref{sec:FeDualEx-convergence}).
    \item FeDualEx produces several byproducts in the CO realm, as demonstrated in Table \ref{tab:composite} : (1) The sequential version of FeDualEx leads to the stochastic dual extrapolation for CO and yields, to our knowledge, the first convergence rate for the stochastic optimization of composite SPP (Section \ref{sec:FeDualEx-stochastic}). (2) Further removing the noise reveals its deterministic version, with rate matching existing ones in smooth and composite saddle point optimization (Section \ref{sec:FeDualEx-deterministic}).
    \item We demonstrate experimentally the effectiveness of FeDualEx on composite saddle point tasks including $\ell_1$ regularization with $\ell_\infty$ ball constraint (Section \ref{sec:experiment}).
\end{itemize}

\begin{table*}[t]
\centering
\begin{adjustbox}{max width=\textwidth}
  \begin{tabular}{c|c|c|c}
    \toprule
    \hline
    {Noise}  & {Rate} &  \cellcolor{green!20} {Composite SPP} & {Smooth SPP }  \\
    \hline
     Deterministic & $\mc{O} \left( \frac{1}{T} \right)$ &  \cellcolor{green!20}\tabularCenterstack{c}{CoMP \citep{he2015mirror} \\
     \textbf{Deterministic FeDualEx (Ours)}} & \tabularCenterstack{c}{Mirror Prox \citep{nemirovski2004prox} \\ Dual Extrapolation \citep{nesterov2007dual} \\ Accelerated Proximal Gradient \citep{tseng2008accelerated} } \\
     \hline
     %\rowcolor{green!20} 
     {Stochastic} & $\mc{O} \left( \frac{1}{\sqrt{T}}\right)$ &  \cellcolor{green!20}\textbf{Sequential FeDualEx (Ours)} & \tabularCenterstack{c}{Stochastic Mirror Prox \citep{juditsky2011solving} \\ \textbf{Sequential FeDualEx (Ours)}} \\
    \hline
    \bottomrule
  \end{tabular}
  \end{adjustbox}
 \caption{Convergence rates for convex-concave SPP. The deterministic version of FeDualEx generalizes dual extrapolation (DE) to composite SPP, and the sequential version of FeDualEx generalizes DE to both smooth and composite stochastic saddle point optimization.} \vspace{-0pt}
 \label{tab:composite} 
\end{table*} 

\vspace{-3pt}
\section{Related Work}
We provide a brief overview of some related work and defer extended discussions to Appendix \ref{appx:review}.

Federated learning was first termed in the algorithm Federated Averaging (FedAvg) \citep{mcmahan17Communication}. \citet{stich2018local} provides the first convergence rate for FedAvg under the homogeneous setting. The rate has been improved with tighter analysis and also analyzed under heterogeneity, to name a few examples \citep{khaled2020tighter, woodworth2020minibatch}. Recently, \cite{yuan2021federated} extended FedAvg to composite convex optimization and proposed FedDualAvg that aggregates learned parameters in the dual space and overcomes the ``curse of primal averaging'' in federated composite optimization.

For SPP, \citet{beznosikov2020distributed} investigate the distributed extra-gradient method for strongly-convex strongly-concave SPP in the Euclidean space. \citet{hou2021efficient} propose FedAvg-S and SCAFFOLD-S based on FedAvg \citep{mcmahan17Communication} and SCAFFOLD \citep{karimireddy2020scaffold} for SPP, which yields similar convergence rate to \citep{beznosikov2020distributed}. Yet, the aforementioned works are limited to smooth and unconstrained SPP in the Euclidean space. 
The more general setting of composite SPP is only found in sequential optimization literature, where the representative composite mirror prox (CoMP) \citep{he2015mirror} generalizes 
the classic mirror prox \citep{nemirovski2004prox} yet keeps the $\mc{O}(\frac{1}{T})$ convergence rate. We will later show that the sequential analysis of our proposed algorithm also yields the same rate for dual extrapolation \citep{nesterov2007dual} in composite optimization, utilizing different proving techniques. And as a result, we focus on the federated learning of composite SPP and propose FeDualEx in this paper.

\section{Preliminaries and Definitions}
We provide some preliminaries and definitions necessary for introducing FeDualEx. More details are included in Appendix \ref{appx:prelim}. We first define the objective: Composite SPP, then briefly review the mirror prox and dual extrapolation as well as techniques for composite convex optimization. We close this section with the basic mechanism of federated learning. To begin with, we lay out the notations.

\textbf{Notations.} We use $[n]$ to represent the set $\{1,2,...,n\}$. We use $\Vert \cdot \Vert$ to denote an arbitrary norm,  $\Vert \cdot \Vert_\ast$ to denote the dual norm, and $\Vert \cdot \Vert_2$ to denote the Euclidean norm. We use $\nabla$ for gradients, $\partial$ for subgradients, and $\langle \cdot, \cdot \rangle$ for inner products. Related to the algorithm, we use English letters (e.g., $z$, $x$, $y$) to denote primal variables, Greek letters (e.g., $\omega$, $\varsigma$, $\mu$, $\nu$) to denote dual variables. We use $R$ for communication rounds, $K$ for local updates, $B$ for diameter bound, $G$ for gradient bound, $\beta$ for smoothness constant, $\sigma$ for standard deviation, $\xi$ for random samples. We use $h^\ast$ to denote the convex conjugate of a function $h$.

\subsection{Composite Saddle Point Optimization}
Due to practical interest and lack of effective methods in FL, we study composite saddle point optimization. Its objective is formally given in the following definition.
\begin{definition}[Composite SPP] \label{def:saddle-obj} The objective of composite saddle point optimization is defined as
    \begin{align} \label{eq:obj}
        \min_{x\in\mc{X}} \max_{y\in\mc{Y}} \phi(x,y) = f(x, y) + \psi_1(x) - \psi_2(y)
    \end{align}
    where $f(x, y) = \frac{1}{M}\sum_{m=1}^Mf_m(x,y)$ and $\psi_1(x)$, $\psi_2(y)$ are possibly non-smooth.
\end{definition}
It is typically evaluated by the duality gap: $\Gap(\hat{x},\hat{y}) = \max_{y\in\mc{Y}}\phi(\hat{x},y) -\min_{x\in\mc{X}}\phi(x,\hat{y})$.

\subsection{Mirror Prox and Dual Extrapolation} \label{sec:prelim-DE}
\begin{wrapfigure}{r}{0.4\textwidth}
    \centering \vspace{-28pt}
    \begin{align*}
        x_t &= \prox{}^h_{\bar{x}} (\mu_t) \\
        x_{t+1/2} &= \prox{}^h_{x_t} (\eta g(x_t)) \\
        \mu_{t+1} &= \mu_t + \eta g(x_{t+1/2})
    \end{align*} \vspace{-15pt}
    \caption{Dual Extrapolation.}
    \label{fig:DualEx} \vspace{5pt}
\end{wrapfigure} 
Mirror prox  \citep{nemirovski2004prox} and dual extrapolation \citep{nesterov2007dual} are classic methods for convex-concave SPP. Both are proximal algorithms based on the proximal operator defined as $$\prox{}^h_{x'}(\cdot) = \argmin_x\{\langle \cdot, x\rangle + V^h_{x'}(x)\},$$ in which $V^h_{x'}(x)$ is the Bregman divergence generated by some closed, strictly convex, and differentiable function $h$, and is defined as follows:
$$V^h_{x'}(x) = h(x) - h(x') - \langle\nabla h(x'), x-x'\rangle.$$
Both algorithms conduct two evaluations of the proximal operator, while dual extrapolation carries out updates in the dual space. Figure \ref{fig:DualEx} gives a brief illustration of dual extrapolation with the proximal operator as in \citep{CohenST21}, with details in Appendix \ref{appx:prelim}.

\subsection{Generalized Bregman Divergence}
Recent advances in composite convex optimization \citep{yuan2021federated} have utilized the Generalized Bregman Divergence \citep{flammarion2017stochastic} for analyzing composite objectives. It incorporates the composite term into the distance-generating function of the vanilla Bregman divergence, and measures the distance in terms of one variable and the dual image of the other, with the key insight being the conjugate of a non-smooth generalized distance-generating function is differentiable. 
\begin{definition}[Generalized Bregman Divergence \citep{flammarion2017stochastic}] Generalized Bregman divergence is defined to be
    $\tilde{V}^{h_t}_{\mu'}(x) = h_t(x) - h_t(\nabla h_t^\ast(\mu')) - \langle \mu', x - \nabla h_t^\ast(\mu')\rangle$, 
where $h_t = h + t\eta\psi$ is a generalized distance-generating function that is closed and strictly convex, $t$ is the current number of iterations, $\eta$ is the step size, $h^\ast_t$ is the convex conjugate of $h_t$, and $\mu'$ is the dual image of $x'$, i.e., $\mu' \in \partial h_t (x')$ and $x' = \nabla h^\ast_t (\mu')$.
\end{definition} \vspace{-5pt}
Generalized Bregman divergence is suitable not only for non-smooth regularization but also for any convex constraints $\mc{C}$, taking $\psi(x) = \begin{cases}
    0 \text{\qquad if } x \in \mc{C} \\
    +\infty \text{ \ otherwise}
\end{cases}.$

\subsection{Federated Learning}
\begin{wrapfigure}{r}{0.4\textwidth}
\hspace{4pt}
  \begin{minipage}{\linewidth}
  \vspace{-28pt}
\setcounter{algorithm}{-1}
\begin{algorithm}[H]
   \caption{Typical FL Procedure}  
   \label{alg:fed-typical}
\begin{algorithmic}[1]
   \FOR{$r = 0,1,\dots,R-1$}
       \STATE Sample a subset of clients
       \STATE Distribute global model to clients
       \FOR{each client \textbf{in parallel}}
            \FOR{$k = 0,1,\dots,K-1$}
               \STATE Certain optimization update
           \ENDFOR
           \STATE \textbf{end for}
           \STATE Send local model to the server
       \ENDFOR
       \STATE \textbf{end parallel for}
       \STATE Server aggregates client models
   \ENDFOR
   \STATE \textbf{end for}
\end{algorithmic}
\end{algorithm}
\vspace{-30pt}
\end{minipage}
\end{wrapfigure}
Federated Learning is a novel distributed learning paradigm where a central server coordinates collaborative learning among clients through rounds of communication. In each round, the server synchronizes the clients with the current global model. Each client participating in this round optimizes the model locally, possibly for several steps, without sharing data, then sends the model to the server. The server then aggregates the models from clients, usually through averaging \citep{stich2018local}, and produces a new global model. The local optimization algorithms can vary based on the objective of interest. This typical procedure is followed by many \citep{mcmahan17Communication, yuan2021federated}, FeDualEx included, and is summarized in Algorithm \ref{alg:fed-typical}.

\section{Federated Dual Extrapolation (FeDualEx)}
In this section, we give our solution to the federated learning of composite saddle point problems. We first present the FeDualEx algorithm and several relevant novel definitions we proposed for its adaptation to composite SPP. As a preview, FeDualEx is presented in Algorithm \ref{alg:fed-DualEx}. Then we analyze the convergence rate for FeDualEx. 

\subsection{The FeDualEx Algorithm} \label{sec:FeDualEx-alg}
To tackle composite SPP in the FL paradigm, we acknowledge the challenges from two aspects. The first comes from composite optimization, which is by itself a complication in sequential saddle point optimization, even convex optimization. The second rises for federated learning, where communication and aggregation need to be carefully handled under the distributed mechanism. In particular, \citet{yuan2021federated} identified the ``the curse of primal averaging'' in composite federated optimization and advocates for dual aggregation. 

With this inspiration, FeDualEx builds its core on the classic dual extrapolation algorithm geared for saddle point optimization. Its effectiveness has been widely verified in vanilla smooth convex-concave SPP. Furthermore, its updating sequence lies in the dual space which would naturally inherit the advantage of dual aggregation in composite federated optimization. 
The challenge remains for composite optimization, as relevant work is limited. The smooth analysis of dual extrapolation is already non-trivial \citep{nesterov2007dual}, and no attempts were previously made for generalizing dual extrapolation to the composite optimization realm.

Further inspired by recent advances in composite convex optimization, we recognize the Generalized Bregman Divergence \citep{flammarion2017stochastic} as a powerful tool for analyzing proximal methods for composite objectives. A detailed introduction is provided in Appendix \ref{appx:prelim}. 

Adapting to the context of composite SPP, we make a further extension to the Generalized Bregman Divergence for saddle functions, and provide the definition below.
\begin{definition}[Generalized Bregman Divergence for Saddle Functions] \label{def:generalized-bregman} The generalized distance-generating function for the optimization of \eqref{eq:obj} is
    $
        \ell_t(z) = \ell(z) + t \eta \psi(z), 
    $
    where $\ell(z) = h_1(x) + h_2(y)$, $\psi(z) = \psi_1(x) + \psi_2(y)$, $\eta$ is the step size, and $t$ is the current number of iterations.
It generates the following generalized Bregman divergence:
    \begin{align*}
        \tilde{V}^{\ell_t}_{\varsigma'}(z) = \ell_t(z) - \ell_t(z') - \langle \varsigma', z - z'\rangle, 
    \end{align*}
    where $\varsigma'$ is the preimage of $z'$ with respect to the gradient of the conjugate of $\ell_t$, i.e., $z' = \nabla\ell_t^\ast(\varsigma')$.
\end{definition}
Yet as we notice in previous works \citep{flammarion2017stochastic, yuan2021federated}, generalized Bregman divergence is applied only for theoretical analysis. In terms of algorithm design, the previous proximal operator for composite convex optimization is based on the vanilla Bregman divergence plus the composite term, specifically, $\argmin_x\{\langle \cdot, x\rangle + V^h_{x'}(x) + \eta\psi(x)\}$ in \citep{duchi2010composite, he2015mirror}, and $\argmin_x\{\langle \cdot, x\rangle + h(x) + \eta t \psi(x)\}$ in \citep{xiao2010dual, flammarion2017stochastic}. However, we find this definition insufficient for dual extrapolation, as its dual update and the composite term from the extra step break certain parts of the analysis.
In this effort, we propose a novel technical change to the proximal operator, directly replacing the Bregman divergence in the proximal operator with the generalized Bregman divergence.
\begin{definition}[Generalized Proximal Operator for Saddle Functions] \label{def:generalized-prox}
    A proximal operation in the composite setting with generalized Bregman divergence for Saddle Functions is defined to be \vspace{-2pt}
    \begin{align*}
        \tilde{\prox{}}_{\varsigma'}^{\ell_t}(g) &\coloneqq \argmin_z \{ \langle g, z \rangle + \tilde{V}^{\ell_t}_{\varsigma'}(z) \},
    \end{align*}  \vspace{-5pt}
where $\varsigma'$ is the dual image of $z'$, i.e., $z' = \nabla \ell_t^\ast (\varsigma')$, and $\varsigma' \in \partial \ell_t(z') = \nabla \ell(z') + \partial \psi(z')$. 
\end{definition}
Compared with the vanilla proximal operator in Section \ref{sec:prelim-DE}, this novel design for the composite adaptation of dual extrapolation is quite natural. It is different from previous proximal operators,  which after expanding take the form $\argmin_z\{\langle \cdot - \nabla \ell(z'), z\rangle + \ell_t(z)\}$ \citep{duchi2010composite} or $\argmin_z\{\langle \cdot, z\rangle + \ell_t(z)$ \citep{xiao2010dual}, whereas ours is $\tilde{\prox{}}^h_{\varsigma'}(\cdot) = \argmin_z\{\langle \cdot - \varsigma', z\rangle + \ell_t(z)\}$.

With the novel definitions above, we are able to formally present FeDualEx in Algorithm \ref{alg:fed-DualEx}. It follows the general structure of FL as in Algorithm \ref{alg:fed-typical}. For each client, the two-step evaluation of the generalized proximal operator and the final dual update are highlighted in {\colorbox{mygreen!30}{green}}, which resembles the classic dual extrapolation updates in Figure \ref{fig:DualEx}. To align with our generalized proximal operator, we also move the primal initialization $\bar{x}$ in the original dual extrapolation to the dual space as $\bar{\varsigma}$. On the server, the dual variables from clients are aggregated first in the dual space, then projected to the primal with a mechanism later defined in \eqref{eq:def-zhat}.

\begin{algorithm}[t]
   \caption{\textsc{Federated-Dual-Extrapolation} (FeDualEx) for Composite SPP} 
   \label{alg:fed-DualEx}
\begin{algorithmic}[1]
    \vspace{-1pt} \REQUIRE $\phi(z) = f(x,y) + \psi_1(x) - \psi_2(y) = \frac{1}{M}\sum_{m=1}^M f_m(\cdot) + \psi_1(x) - \psi_2(y)$: objective function; $\ell(z)$: distance-generating function; $g_m(z) = (\nabla_x f_m (x,y), -\nabla_y f_m (x,y))$: gradient operator.
    \renewcommand{\algorithmicrequire}{\textbf{Hyperparameters:}} \REQUIRE $R$: number of communication rounds; $K$: number of local update iterations; $\eta^s$: server step size; $\eta^c$: client step size.
    \renewcommand{\algorithmicrequire}{\textbf{Dual Initialization:}} \REQUIRE $\varsigma_{0} = 0$: initial dual variable, $\bar{\varsigma}$: fixed point in the dual space.
    \ENSURE Approximate solution $z = (x, y)$ to $ \min_{x \in \mc{X}} \max_{y \in \mc{Y}} \phi(x, y)$ 
   \FOR{$r = 0,1,\dots,R-1$}
       \STATE Sample a subset of clients $C_r \subseteq [M]$
       \FOR{$m \in C_r$ \textbf{in parallel}}
           \STATE $\varsigma^m_{r,0} = \varsigma_{r}$
           \tikzmk{A}\FOR{$k = 0,1,\dots,K-1$}
               \STATE $z^m_{r,k} = \tilde{\prox{}}_{\bar{\varsigma}}^{\ell_{r,k}} (\varsigma_{r,k}^m)$ \hfill $\vartriangleright$ Two-step evaluation of the generalized proximal operator
               \STATE $z^m_{r,k+1/2} = \tilde{\prox{}}_{\bar{\varsigma} - \varsigma_{r,k}^m}^{\ell_{r,k+1}} (\eta^c g_m(z_{r,k}^m; \xi^m_{r,k}))$
               \STATE $\varsigma^m_{r,k+1} = \varsigma^m_{r,k} + \eta^c g_m(z^m_{r,k+1/2}; \xi^m_{r,k+1/2})$ \hfill $\vartriangleright$ Dual variable update
           \ENDFOR
           \STATE \textbf{end for}
       \ENDFOR
       \STATE \textbf{end parallel for}
       \tikzmk{B}  \boxit{mygreen}
       \STATE $\Delta_r = \frac{1}{|\mc{C}_r|}\sum_{m\in \mc{C}_r} (\varsigma^m_{r,K} - \varsigma^m_{r,0})$
       \STATE $\varsigma_{r+1} = \varsigma_r + \eta^s \Delta_r$ \hfill $\vartriangleright$ Server dual update
   \ENDFOR
   \STATE \textbf{end for}
   \STATE \textbf{Return:} $\frac{1}{RK} \sum_{r=0}^{R-1}\sum_{k=0}^{K-1} \widehat{z_{r,k+1/2}}$ with $\widehat{z_{r,k+1/2}}$ defined in \eqref{eq:def-zhat}.
\end{algorithmic} \vspace{1pt}
\end{algorithm}

\subsection{Convergence Analysis of FeDualEx} \label{sec:FeDualEx-convergence}
In this section, we provide the convergence analysis of FeDualEx for the homogeneous FL of composite SPP.  

We further assume the full participation of clients in each round for simplicity, but this condition can be trivially removed by lengthy analysis.  We start by listing the key assumptions. Detailed presentation and additional remarks that ease the understanding of proofs are also provided in Appendix \ref{appx:asm}. 

\textbf{\hypertarget{Assumptions}{Assumptions.}} \label{asm:simple} \textit{For the composite saddle function $\phi(x,y) = \frac{1}{M}\sum_{m=1}^M f_m(x,y) + \psi_1(x) - \psi_2(y)$, its gradient operator is given by $g= \left( \nabla_x f, -\nabla_yf\right)$ and $g = \frac{1}{M} \sum_{m=1}^M g_m$. We assume that \vspace{-5pt} %\setlength\itemsep{-1pt}
    \begin{itemize}
        \item [a.] (Convexity of $f$) $\forall m \in [M]$, $f_m(x, y)$ is convex in $x$ and concave in $y$.
        \item [b.] (Convexity of $\psi$) $\psi_1(x)$ is convex in $x$, and $\psi_2(y)$ is convex in $y$.
        \item [c.] (Lipschitzness of $g$) $g_m(z) = \left[ \begin{smallmatrix}\nabla_x f_m(x, y) \\ -\nabla_yf_m(x, y)\end{smallmatrix}\right]$ is $\beta$-Lipschitz:
        $$
            \nrm{g_m(z) - g_m(z')}_{\ast} \leq \beta \nrm{z - z'}
        $$
        \item [d.] (Unbiased Estimate and Bounded Variance) $\forall m \in [M]$, for random sample $\xi^m$, 
        \begin{align*}
            \mathbb{E}_\xi [g_m(z^m;\xi^m)] = g_m(z^m), && \mathbb{E}_\xi \big[\nrm{g_m(z^m;\xi^m) - g_m(z^m)}_\ast^2 \big] \leq \sigma^2.
        \end{align*}
        \item [e.] (Bounded Gradient) $\forall m \in [M]$,
        $
        \nrm{g_m(z^m;\xi^m)}_\ast \leq G
        $
        \item [f.] The distance-generating function $\ell$ is a Legendre function that is 1-strongly convex, i.e., $\forall z, z'$, 
    \begin{align*}
        \ell(z') - \ell(z) - \langle \nabla \ell(z), z'-z \rangle \geq \frac{1}{2} \nrm{z'-z}^2.
    \end{align*} 
    \item [g.] The optimization domain $\mc{Z}$ is compact w.r.t. Bregman divergence, i.e., $\forall z, z' \in \mc{Z}$, $V_{z'}^{\ell}(z) \leq B$.
    \end{itemize}
}

Next, we show the equivalence between primal-dual projection, also known as the mirror map, and the generalized proximal operator, and for the convenience of analysis, reformulate the updating sequences with another pair of auxiliary dual variables.

\textbf{Projection Reformulation.} Generalized proximal operators can be presented as projections, i.e., the gradient of the conjugate of the generalized distance-generating function in Appendix \ref{appx:def}. Thus, line 6 to 8 in Algorithm \ref{alg:fed-DualEx} can be expanded by Definition \ref{def:generalized-prox}, and rewrite as: 
\begin{align*}
     z^m_{r,k} &= \nabla \ell_{r,k}^\ast(\bar{\varsigma} - \varsigma_{r,k}^m); \\
     z^m_{r,k+1/2} &= \nabla \ell_{r,k+1}^\ast((\bar{\varsigma} - \varsigma_{r,k}^m) - \eta^c g_m(z_{r,k}^m; \xi^m_{r,k})); \\
     \varsigma^m_{r,k+1} &= \varsigma^m_{r,k} + \eta^c g_m(z^m_{r,k+1/2}; \xi^m_{r,k+1/2}).
\end{align*}
Further define auxiliary dual variable $\omega^m_{r,k} = \bar{\varsigma} - \varsigma_{r,k}^m$. It satisfies immediately that $z^m_{r,k} = \nabla \ell_{r,k}^\ast(\omega^m_{r,k})$, in which $\ell_{r,k}^\ast$ is the conjugate of $\ell_{r,k} = \ell + (\eta^srK+k)\eta^c\psi$. And define $\omega^m_{r,k+1/2}$ to be the dual image of the intermediate variable $z^m_{r,k+1/2}$ such that $z^m_{r,k+1/2} = \nabla \ell_{r,k+1}^\ast(\omega^m_{r,k+1/2})$. Then we get an equivalent updating sequence with the auxiliary dual variables.
\begin{align*}
    \omega^m_{r,k+1/2} &= {\omega^m_{r,k}} - \eta g_m(z^m_{r,k}; \xi^m_{r,k}), \\
    \omega^m_{r,k+1} &= \omega^m_{r,k} - \eta g_m(z^m_{r,k+1/2}; \xi^m_{r,k+1/2})
\end{align*}
Define their average across clients,
$\overline{\omega_{r,k}} = \frac{1}{M}\sum_{m=1}^M\omega^m_{r,k}$,
$ \overline{g_{r,k}} = \frac{1}{M}\sum_{m=1}^M g_m(z^m_{r,k}; \xi^m_{r,k})$.
Then we can analyze the following averaged dual shadow sequences:
\begin{align}
    \overline{\omega_{r,k+1/2}} &= \overline{\omega_{r,k}} - \eta^c \overline{g_{r,k}}, \label{eq:seq1}\\
    \overline{\omega_{r,k+1}} &= \overline{\omega_{r,k}} - \eta^c \overline{g_{r,k+1/2}}. \label{eq:seq2}
\end{align}
In the meantime, their shadow primal projections on the server are defined as
\begin{align} 
    \widehat{z_{r, k}} = \nabla \ell_{r,k}^\ast (\overline{\omega_{r,k}}) , \qquad \qquad \widehat{z_{r, k+1/2}} = \nabla \ell_{r,k+1}^\ast (\overline{\omega_{r,k+1/2}}). \label{eq:def-zhat}
\end{align}

\textbf{Main Theorem.} Under the aforementioned assumptions, we present the following theorem that provides the convergence rate of FeDualEx in terms of the duality gap.
\begin{restatable}[Main]{theorem}{thmmain} \label{thm:main} Under \hyperlink{Assumptions}{assumptions}, the duality gap evaluated with the ergodic sequence generated by the intermediate steps of FeDualEx in Algorithm \ref{alg:fed-DualEx} is bounded by 
\begin{align*}
    \mathbb{E}\Big[\Gap \Big (\frac{1}{RK}\sum_{r=0}^{R-1}\sum_{k=0}^{K-1}\widehat{z_{r, k+1/2}} \Big)\Big] &\leq \frac{B}{\eta^cRK} + 20\beta^2(\eta^c)^3K^2G^2 + \frac{5\sigma^2\eta^c}{M} + 2^\frac{3}{2}\beta\eta^cKGB.
\end{align*} 
Choosing step size
$
    \eta^c = \min \{\frac{1}{5\beta^2}, \frac{B^\frac{1}{4}}{20^\frac{1}{4}\beta^\frac{1}{2}G^\frac{1}{2}K^\frac{3}{4}R^\frac{1}{4}}, \frac{B^\frac{1}{2}M^\frac{1}{2}}{5^\frac{1}{2}\sigma R^\frac{1}{2} K^\frac{1}{2}}, \frac{1}{2^\frac{3}{4}\beta^\frac{1}{2}G^\frac{1}{2}KR^\frac{1}{2}}\}
$, 
\begin{align*}
    \mathbb{E}\Big[\Gap\Big(\frac{1}{RK}\sum_{r=0}^{R-1}\sum_{k=0}^{K-1}\widehat{z_{r, k+1/2}}\Big)\Big] &\leq \frac{5\beta^2B}{RK} + \frac{20^\frac{1}{4}\beta^\frac{1}{2}G^\frac{1}{2}B^\frac{3}{4}}{K^\frac{1}{4}R^\frac{3}{4}} + \frac{5^\frac{1}{2}\sigma B^\frac{1}{2}}{M^\frac{1}{2}R^\frac{1}{2} K^\frac{1}{2}} + \frac{2^\frac{3}{4}\beta^\frac{1}{2}G^\frac{1}{2}B}{R^\frac{1}{2}}.
\end{align*}
\end{restatable} 
To the best of our knowledge, this is the first convergence rate for federated composite saddle point optimization. The $\mc{O}(\frac{1}{RK})$ and $\mc{O}(\frac{1}{\sqrt{MRK}})$ terms roughly match previous FL algorithms with a $\mc{O}(1/R^\frac{1}{2})$ term taking domination in terms of communication complexity assuming the number of clients is large enough. The convergence analysis further validates the effectiveness of FeDualEx, which then advances federated learning to a broad class of composite saddle point problems.

\textbf{Outline of Proof Technique.} We provide the proof sketch to Theorem \ref{thm:main} with two key lemmas, and provide the complete proof in Appendix \ref{appx:FeDualEx-Saddle}. The core idea is to upper bound the duality gap with the smooth term $f$ and the composite possibly non-smooth regularization term $\psi$ separately. Similar ideas are applied for analyzing composite convex optimization \citep{flammarion2017stochastic, yuan2020federated}. The non-smooth term is bounded in Lemma \ref{lem:non-smooth}, whose proof relies on generating the regularization term with the generalized Bregman divergence and is deferred to Appendix \ref{appx:FeDualEx-Saddle}.
\begin{restatable}[Bounding the Regularization Term]{lemma}{lemnonsmooth} \label{lem:non-smooth} Under the same assumption as Theorem \ref{thm:main}, $\forall z \in \mc{Z}$, 
\begin{align*}
 \eta^c \big[ \psi(\widehat{z_{r, k+1/2}}) - \psi(z)\big] &= \tilde{V}_{\overline{\omega_{r,k}}}^{\ell_{r,k}}(z)  - \tilde{V}_{\overline{\omega_{r,k+1}}}^{\ell_{r,k+1}}(z) - \tilde{V}_{\overline{\omega_{r,k}}}^{\ell_{r,k}}(\widehat{z_{r, k+1/2}}) - \tilde{V}_{\overline{\omega_{r,k+1/2}}}^{\ell_{r,k+1}}(\widehat{z_{r, k+1}}) \\ 
    & \quad + \eta^c\langle \overline{g_{r,k+1/2}} - \overline{g_{r,k}}, \widehat{z_{r, k+1/2}} - \widehat{z_{r, k+1}} \rangle + \eta^c\langle \overline{g_{r,k+1/2}}, z - \widehat{z_{r, k+1/2}} \rangle.
\end{align*}
\end{restatable} 
This lemma breaks the bound for the non-smooth regularization into four generalized Bregman divergence terms, in which the first two are ready for telescoping. The last generalized Bregman divergence and the following inner product are generated due to the extra-step of FeDualEx. The final term is to be canceled with one term in the smooth bound.
\begin{restatable}[Bounding the Smooth Term]{lemma}{lemsmooth} \label{lem:smooth}
Under the same assumption as Theorem \ref{thm:main}, $\forall z \in \mc{Z}$, 
{
\begin{align*} 
     \langle g(\widehat{z_{r, k+1/2}}), \widehat{z_{r, k+1/2}} - z\rangle &= \langle \overline{g_{r,k+1/2}}, \widehat{z_{r, k+1/2}} - z\rangle + \langle \frac{1}{M} \sum_{m=1}^M g_m(z_{r,k+1/2}^m) - \overline{g_{r,k+1/2}}, \widehat{z_{r, k+1/2}} - z\rangle \\ 
    &\quad + \langle \frac{1}{M} \sum_{m=1}^M g_m(z_{r,k+1/2}^m) - \overline{g_{r,k+1/2}}, \widehat{z_{r, k+1/2}} - z\rangle
\end{align*}} 
\end{restatable}
Summing Lemma \ref{lem:non-smooth} and Lemma \ref{lem:smooth} yields the per-step progress for FeDualEx, with some remaining terms that further generate conventional terms in FL like client drift and deviation, and are to be bounded with helping lemmas in Appendix \ref{appx:FeDualEx-Saddle}. After telescoping, we retrieve the result in Theorem \ref{thm:main}.

\textbf{On Composite Convex Optimization.} We also analyze the convergence rate for FeDualEx under the federated composite convex optimization setting. As the following theorem shows, FeDualEx achieves the same $\mc{O}(1/R^\frac{2}{3})$ as in \citep{yuan2021federated}. The proof is provided in Appendix \ref{appx:FeDualEx-Convex}. 
\begin{restatable}{theorem}{thmconvex} \label{thm:convex} Under the convex counterparts of previous assumptions, choosing step size 
\begin{align*}
    \eta^c = \min \{\frac{1}{5\beta^2}, \frac{B^\frac{1}{4}}{20^\frac{1}{4}\beta^\frac{1}{2}G^\frac{1}{2}K^\frac{3}{4}R^\frac{1}{4}}, \frac{B^\frac{1}{2}M^\frac{1}{2}}{5^\frac{1}{2}\sigma R^\frac{1}{2} K^\frac{1}{2}}, \frac{B^\frac{1}{3}}{2^\frac{1}{3}\beta^\frac{1}{3}G^\frac{2}{3}KR^\frac{1}{3}}\},
\end{align*} 
the ergodic intermediate sequence generated by FeDualEx for composite convex objectives satisfies 
\begin{align*}
    \mathbb{E}\big[\phi(\frac{1}{RK}\sum_{r=0}^{R-1}\sum_{k=0}^{K-1}\widehat{x_{r, k+1/2}}) - \phi(x) \big] &\leq \frac{5\beta^2B}{RK} + \frac{20^\frac{1}{4}\beta^\frac{1}{2}G^\frac{1}{2}B^\frac{3}{4}}{K^\frac{1}{4}R^\frac{3}{4}} + \frac{5^\frac{1}{2}\sigma B^\frac{1}{2}}{M^\frac{1}{2}R^\frac{1}{2} K^\frac{1}{2}} + \frac{2^\frac{1}{3}\beta^\frac{1}{3}G^\frac{2}{3}B^\frac{2}{3}}{R^\frac{2}{3}}.
\end{align*}
\end{restatable} 
Even though this rate is not preserved in composite saddle point optimization, we note that the optimization of SPP is much more general, and convexity itself is a stronger assumption. More specifically, the complicated setting, including the non-smooth term, the primal-dual projection, the extra-step saddle point optimization, etc., together limit the tools available for analysis. We leave possible improvements as future work. 

\textbf{Remark On Heterogeneity.} Even for federated composite optimization \citep{yuan2021federated}, the heterogeneous setting presents significant hurdles. Specifically, the involvement of heterogeneity is limited to quadratic functions, under which assumption the is gradient linear, and this simplifies the analysis. It further relies on the norm generated by its Hessian. For saddle functions, ``quadraticity'' (as well as a matrix-induced norm) is less well-defined, as the Jacobian of their gradient operator is not (symmetric) positive semidefinite in general. Such further advancements go beyond the scope of this paper. Thus, we regard the rate in Theorem \ref{thm:main} as a significant start for federated composite saddle point optimization.

\section{FeDualEx in Sequential Settings}
In this section, we briefly exhibit the results that come naturally by applying FeDualEx to sequential settings in the composite optimization realm, namely stochastic and deterministic composite saddle point optimization.
\subsection{Stochastic Composite Saddle Point Optimization} \label{sec:FeDualEx-stochastic}
FeDualEx can be naturally reduced to sequential stochastic optimization of composite SPP. We term this algorithm \textit{Sequential FeDualEx} or \textit{Stochastic Dual Extrapolation}. Relevant algorithms or theoretical convergence rates under the same setting, to the best of our knowledge, are not found in prior literature. By reducing the number of clients $M$ to one, thus eliminating the need for communication, and further denoting the local updates $K$ as general iterations $T$, the convergence analysis follows through smoothly and yields $\mc{O}(\frac{1}{\sqrt{T}})$ rate expected for first-order stochastic algorithms by the following theorem. The proof can be found in Appendix \ref{appx:stochastic}.
\begin{theorem} \label{thm:stochastic} Under the sequential versions of previous assumptions, $\forall z \in \mc{Z}$, choosing step size
$\eta = \min \{\frac{1}{3\beta^2}, \frac{B^\frac{1}{2}}{3^\frac{1}{2}\sigma T^\frac{1}{2}}\}$, the ergodic intermediate sequence of stochastic dual extrapolation satisfies 
\begin{align*}
    \mathbb{E}\big[\phi(\frac{1}{T}\sum_{t=0}^{T-1}z_{t+1/2}) - \phi(z) \big] &\leq \frac{3\beta^2B}{T} + \frac{3^\frac{1}{2}\sigma B^\frac{1}{2}}{T^\frac{1}{2}}.
\end{align*}
\end{theorem} 

\subsection{Deterministic Composite Saddle Point Optimization} \label{sec:FeDualEx-deterministic}
Further removing the noise in gradient, FeDualEx reduces to a deterministic algorithm for composite SPP. We emphasize that even so, we are still generalizing the classic dual extrapolation algorithm to composite optimization, and thus term the algorithm \textit{Deterministic FeDualEx} or \textit{Composite Dual Extrapolation}. Following a similar analysis, we are able to get the $\mc{O}(\frac{1}{T})$ rate as in previous work for composite optimization \citep{he2015mirror} as well as the smooth dual extrapolation \citep{nesterov2007dual}. The proof for the following theorem is in Appendix \ref{appx:deterministic}, which is, in particular, a much simpler one as we utilize the recently proposed Relative Lipschitzness condition \citep{CohenST21}.
\begin{restatable}{theorem}{thmdeterministic} \label{thm:deterministic}
    Under the basic convexity assumption and $\beta$-Lipschitzness of $g$, $\forall z \in \mc{Z}$ and $\eta \leq \frac{1}{\beta}$, composite dual extrapolation satisfies 
\begin{align*} 
    \mathbb{E}\big[\phi(\frac{1}{T}\sum_{t=0}^{T-1}z_{t+1/2}) - \phi(z) \big] &\leq \frac{\beta B}{T}.
\end{align*}
\end{restatable}

\section{Experiments} \label{sec:experiment}
\begin{figure}[b!] 
\begin{minipage}{.48\textwidth}
    \centering 
    {\begin{align*}
       \min_{\mathbf{x} \in \mc{X}} \ \max_{\mathbf{y} \in \mc{Y}} \langle {\mathbf{Ax}} - \mathbf{b}, \mathbf{y} \rangle + \lambda {\nrm{\mathbf{x}}}_1 - \lambda \nrm{\mathbf{y}}_1 
    \end{align*} \vspace{-10pt}
    \begin{align*} 
        &{\mathbf{A}} \in \mathbb{R}^{n\times m},
        &\mc{X} =  \{\mathbb{R}^{m}: \Vert \mathbf{x} \Vert_\infty \leq D\}, \\ &{\mathbf{b}} \in \mathbb{R}^{n}, &\mc{Y} =  \{\mathbb{R}^{n}: \Vert \mathbf{y} \Vert_\infty \leq D\}.
    \end{align*} 
    \\
    \noindent
    \rule{0.9\textwidth}{0.4pt} \vspace{5pt}
    \begin{align*}
        &\Gap(\mathbf{x}, \mathbf{y}) = D\nrm{\max\{\vert \mathbf{Ax}-\mathbf{b}\vert - \lambda, 0\}}_1 \\
        & \qquad + \lambda \nrm{\mathbf{x}}_1 + D\nrm{\max\{\vert \mathbf{A}^\top\mathbf{y} \vert - \lambda, 0\}}_1 \\
        & \qquad + \langle \mathbf{b}, \mathbf{y} \rangle + \lambda \nrm{\mathbf{y}}_1.
    \end{align*}}
    \caption{The composite saddle point optimization problem with $\ell_1$ norm sparsity regularization from \citep{jiang2022generalized}, and the evaluation of its duality gap given in the closed-form.}
    \label{fig:saddle-problem}
\end{minipage} 
\hfill
\begin{minipage}{.48\textwidth}
\centering
    \centering 
    {\begin{align*}
       \min_{\mathbf{X} \in \mc{X}} \ \max_{\mathbf{Y} \in \mc{Y}} \mathrm{Tr} \big( (\mathbf{AX} - \mathbf{B})^\top \mathbf{Y} \big) + \lambda {\nrm{\mathbf{X}}}_\ast - \lambda \nrm{\mathbf{Y}}_\ast 
    \end{align*} \vspace{-10pt}
    \begin{align*} 
        & {\mathbf{A}} \in \mathbb{R}^{n\times m},
        \qquad \mc{X} =  \{\mathbb{R}^{m \times p}: \Vert \mathbf{X} \Vert_2 \leq D\}, \\ &{\mathbf{B}} \in \mathbb{R}^{n \times p}, \qquad \  \mc{Y} =  \{\mathbb{R}^{n \times p}: \Vert \mathbf{Y} \Vert_2 \leq D\}.
    \end{align*} 
    \\
    \noindent
    \rule{0.9\textwidth}{0.4pt} \vspace{5pt}
    \begin{align*}
        &\Gap(\mathbf{X}, \mathbf{Y}) = D\nrm{\mathrm{diag} \big((\vert \sigma_i ( \mathbf{AX}-\mathbf{B} ) \vert - \lambda)_+\big)}_\ast \\
        & \qquad + \lambda \nrm{\mathbf{X}}_\ast + D\nrm{\mathrm{diag} \big((\vert \sigma_j (  \mathbf{A}^\top\mathbf{Y} ) \vert - \lambda)_+\big)}_\ast \\
        & \qquad + \mathrm{Tr} \big(\mathbf{B^\top Y}\big) + \lambda \nrm{\mathbf{Y}}_\ast.
    \end{align*}}
    \caption{The composite saddle point optimization problem with nuclear norm low-rank regularization, and the evaluation of its duality gap given in the closed-form.}
    \label{fig:saddle-problem2}
\end{minipage} 
\end{figure}

In this section, we verify the effectiveness of FeDualEx by numerical evaluation. We compare FeDualEx against FedDualAvg and FedMiD \citep{yuan2021federated}, as well as FedMiP proposed in Algorithm \ref{alg:fed-MiP} in Appendix \ref{appx:FedMiP}. We present problem formulations and experiment results here and defer detailed settings to Appendix \ref{appx:exp}.

\subsection{Saddle Point Problem with Sparsity Regularization and Ball Constraint}

We test all methods on the bilinear problem with $\ell_1$ regularization and $\ell_\infty$ ball constraint from \citep{jiang2022generalized}, which is presented in Figure \ref{fig:saddle-problem}.  The purpose of $\ell_1$ regularization is to encourage sparsity. We take the distance-generating function to be $\ell = \frac{1}{2}\Vert\mathbf{x}\Vert_2^2 + \frac{1}{2}\Vert\mathbf{y}\Vert_2^2$, so the generalized proximal operator instantiates to the soft-thresholding operator \citep{hastie2015statistical, jiang2022generalized}. We generate a fixed pair of ${\mathbf{A}}$ and ${\mathbf{b}}$ with each entry independently following the uniform distribution $\mc{U}_{[-1,1]}$. Each entry of the variables ${\mathbf{x}}$ and ${\mathbf{y}}$ is initialized independently from the distribution $\mc{U}_{[-D,D]}$. As in \citep{jiang2022generalized}, we take $m=600$, $n=300$, $\lambda = 0.1$, $D = 0.05$.  For federated learning, we simulate $M=100$ clients. For the gradient query of each client in each local update, we inject a Gaussian noise from $\mc{N}(0, \sigma^2)$.  All $M=100$ clients participate in each round; noise on each client is i.i.d. with $\sigma = 0.1$. 

We evaluate the convergence in terms of the duality gap and also demonstrate the sparsity of the solution. The duality gap for the problem of interest can be evaluated in closed form, which is also provided in Figure \ref{fig:saddle-problem}. The sparsity is measured by the ratio of non-zero entries to the parameter size, and we regard numbers less than $10^{-5}$ as zeros. The evaluation is conducted for two different settings: (a) $K=1$ local update for $R=5000$ rounds; (b) $K=10$ local updates for $R=500$ rounds. The results are demonstrated in Figure \ref{fig:l1_res} correspondingly. 

\textbf{Discussions.} From the duality gap curves, we see that extra-step methods, i.e., FeDualEx and FedMiP converge to the order of $10^{-1}$ whereas FedDualAvg and FedMiD stay above $10^0$. Thus, it is evident that methods for composite convex optimization are no longer suited for composite saddle point optimization, and FeDualEx provides the first effective solution addressing the challenge. From the sparsity of the solution, we see that the dual methods demonstrate better adherence to regularization. Among the methods superior in saddle point optimization, FeDualEx reaches a sparsity of around $0.7$ while FedMiP around $0.95$. This aligns with the previous analysis on the advantage of dual aggregation and further validates the effectiveness of FeDualEx for solving composite SPP.

\begin{figure*}[t]
\centering
        \begin{subfigure}{0.49\textwidth}
        \centering
        \includegraphics[width=\textwidth]{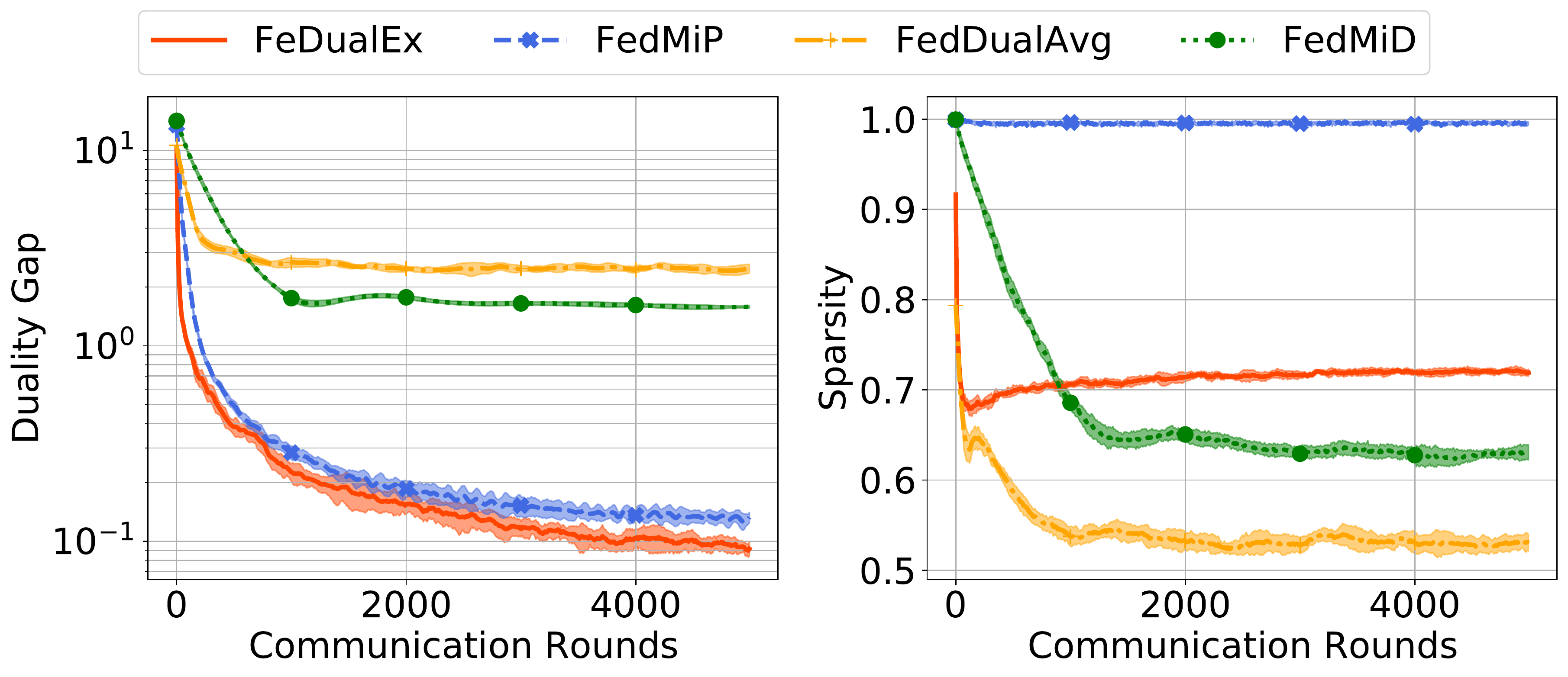}
        \caption{One Local Update} 
        \end{subfigure}
        \hfill
        \begin{subfigure}{0.49\textwidth}
        \centering
        \includegraphics[width=\textwidth]{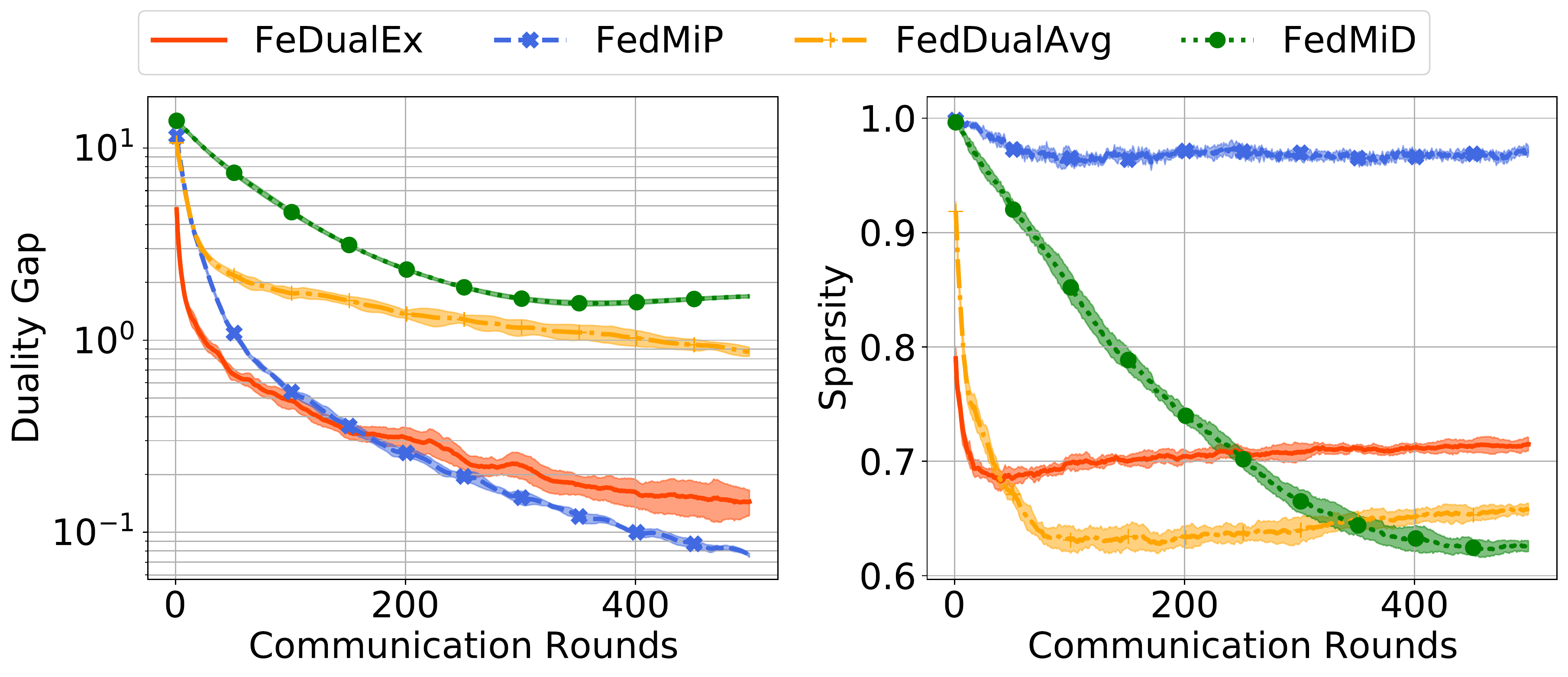}
        \caption{Ten Local Updates}
        \end{subfigure}%    
        \caption{Duality gap and sparsity of the solution to the SPP in Figure \ref{fig:saddle-problem}.}
        \label{fig:l1_res}
\end{figure*} 

\subsection{Saddle Point Problem with Nuclear Norm Regularization and Spectral Norm Constraint}

\begin{figure*}[t]
\centering
        \begin{subfigure}{\textwidth}
        \centering
        \includegraphics[width=\textwidth]{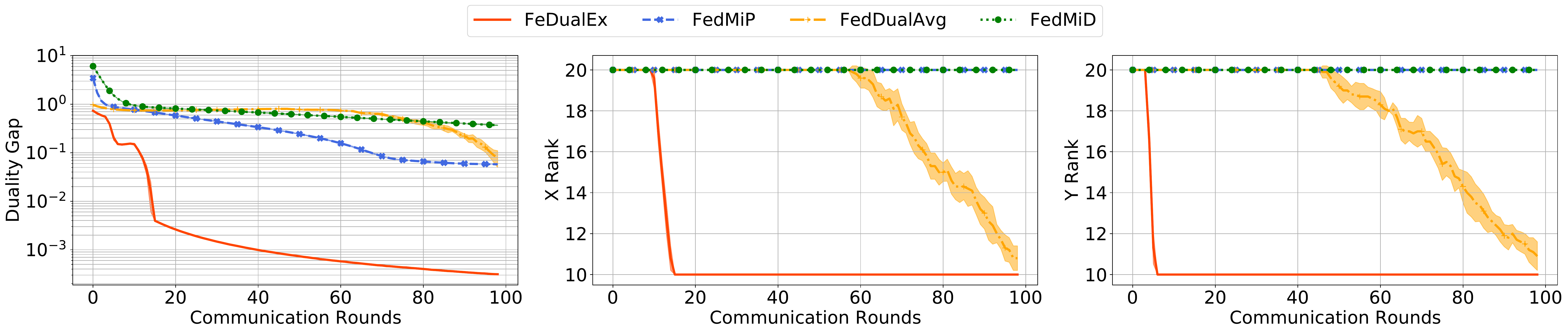}
        \caption{One Local Update} \vspace{5pt}
        \end{subfigure}
        \begin{subfigure}{\textwidth}
        \centering
        \includegraphics[width=\textwidth]{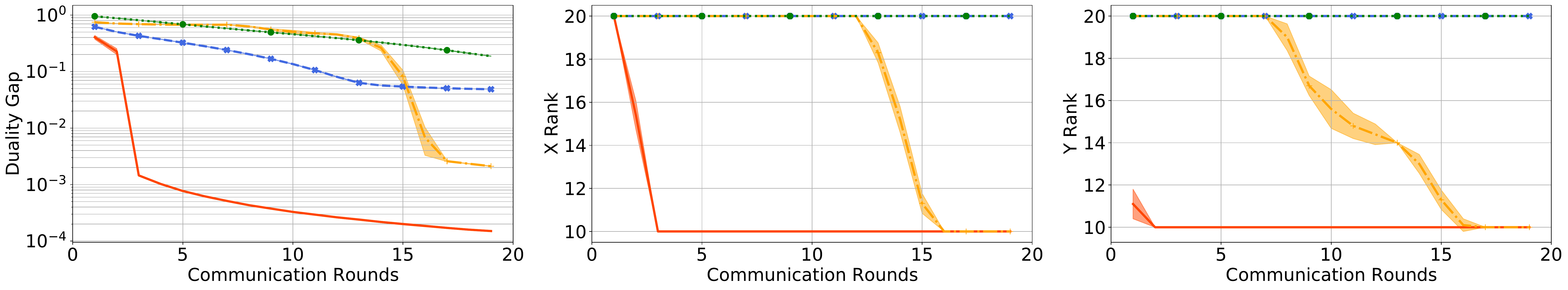}
        \caption{Ten Local Updates}
        \end{subfigure}%    
        \caption{Duality gap and rank of the solution to the nuclear norm regularized SPP in Figure \ref{fig:saddle-problem2}.}
        \label{fig:nuc20_res}
\end{figure*}

We also test FeDualEx on the SPP with nuclear norm regularization for low-rankness, as shown in Figure \ref{fig:saddle-problem2}, in which we overuse the notation $\Vert \cdot \Vert_\ast$ for the matrix nuclear norm and $\Vert \cdot \Vert_2$ for the matrix spectral norm. We use $\mathrm{Tr}(\cdot)$ to denote the trace of a square matrix. And for the purpose of feasibility and convenience, we impose spectral norm constraints on the variables as well. By choosing the distance-generating function to be $\ell = \frac{1}{2}\Vert\mathbf{X}\Vert_\mathrm{F}^2 + \frac{1}{2}\Vert\mathbf{Y}\Vert_\mathrm{F}^2$ where $\Vert \cdot \Vert_\mathrm{F}$ denotes the Frobenius norm, the projection $\nabla \ell_{r,k}^\ast (\cdot)$ instantiates to the singular value soft-thresholding operator \citep{cai2010singular}. 

The data-generating process is similar to that in the previous SPP. The key difference is, for the feasibility of low-rankness, we generate ${\mathbf{B}}$ to be of rank $\frac{p}{2}$, i.e. half of the columns of $B$ is linearly dependent on the other half. We take $p=20$, so the optimal rank for the solution would most likely be $10$.

We evaluate the convergence in terms of the duality gap and also demonstrate the rank of the solution, for both $\mathbf{X}$ and $\mathbf{Y}$. The duality gap can be evaluated in closed form as presented in Figure \ref{fig:saddle-problem2}. The evaluation is conducted for two different settings: (a) $K=1$ local update for $R=100$ rounds; (b) $K=10$ local updates for $R=20$ rounds. The results are demonstrated in Figure \ref{fig:nuc20_res} correspondingly.

\textbf{Discussions.} 
  From Figure \ref{fig:nuc20_res}, we can see that in the setting for low-rankness regularization, dual methods tend to perform better both in minimizing the duality gap and in encouraging a low-rank solution. In particular, FeDualEx, as a method geared for saddle point optimization, demonstrates better convergence in the duality gap than FedDualAvg. In the meantime, the solution given by FeDualEx quickly reaches the optimal rank of $10$. This further reveals the potential of FeDualEx in coping with a variety of regularization and constraints.

\section{Conclusion and Future Work}
We advance federated learning to the broad class of composite SPP by proposing FeDualEx and providing, to our knowledge, the first convergence rate of its kind. We also show that the sequential version of FeDualEx provides a solution to composite stochastic saddle point optimization, and such analysis, to our knowledge, was previously not found. We recognize further study of the heterogeneous federated setting of composite saddle point optimization would be a challenging direction for future work.

\bibliography{reference}

\newpage

\appendix
{\Large{\textbf{Appendices}}}

\flushbottom

In Appendix \ref{appx:exp}, we provide details on experiment settings and additional experiments on saddle point optimization with low-rank nuclear norm regularization. In Appendix \ref{appx:review}, an extended literature review on various related subfields is included. Appendix \ref{appx:additional} and \ref{appx:tech-lemmas} provide additional theoretical background, including relevant preliminaries, definitions, remarks, and technical lemmas. Appendix \ref{appx:FeDualEx-Saddle}, \ref{appx:FeDualEx-Convex}, and \ref{appx:FeDualEx-Others} provide the convergence rates and complete proofs for FeDualEx in federated composite saddle point optimization, federated composite convex optimization, sequential stochastic composite optimization, and sequential deterministic composite optimization respectively. Finally, the algorithm of FedMiP is presented in Appendix \ref{appx:FedMiP}.

\startcontents[appendices]
\printcontents[appendices]{l}{1}%{\section{Appendices}\setcounter{tocdepth}{2}}

\newpage
\section{Experiment Setup Details} \label{appx:exp}

\subsection{Setup Details for Saddle Point Optimization with Sparsity Regularization}
We provide additional details for the SPP with the sparsity regularization demonstrated in the main text. We start by restating its formulation below:
\begin{align*}
       &\min_{\mathbf{x} \in \mc{X}} \ \max_{\mathbf{y} \in \mc{Y}} \langle {\mathbf{Ax}} - \mathbf{b}, \mathbf{y} \rangle + \lambda {\nrm{\mathbf{x}}}_1 - \lambda \nrm{\mathbf{y}}_1 \\ 
        &{\mathbf{A}} \in \mathbb{R}^{n\times m},
        \qquad \mc{X} =  \{\mathbb{R}^{m}: \Vert \mathbf{x} \Vert_\infty \leq D\}, \\ &{\mathbf{b}} \in \mathbb{R}^{n}, \qquad \quad \ \  \mc{Y} =  \{\mathbb{R}^{n}: \Vert \mathbf{y} \Vert_\infty \leq D\}.
\end{align*} 

\textbf{Soft-Thresholding Operator for $\ell_1$ Norm Regularization.} By choosing the distance-generating function to be $\ell = \frac{1}{2}\Vert\mathbf{x}\Vert_2^2 + \frac{1}{2}\Vert\mathbf{y}\Vert_2^2$, the projection $\nabla \ell_{r,k}^\ast (\cdot)$ instantiates to the following element-wise soft-thresholding operator \citep{hastie2015statistical, jiang2022generalized}: 
\begin{align*}
    T_{\lambda'}(\omega) \coloneqq 
    \begin{cases}
      0 & \text{if } \vert \omega \vert \leq \lambda' \\
      (\vert \omega \vert - \lambda') \cdot \mathrm{sgn}(\omega) & \text{if } \lambda' < \vert \omega \vert \leq  \lambda' + D \\
      D \cdot \mathrm{sgn}(\omega) & \text{otherwise}
    \end{cases}, 
\end{align*}
in which $\lambda' = \lambda \eta^c (\eta^s r K + k)$.

\textbf{Closed-Form Duality Gap.} The closed-form duality gap is given by 
\begin{align*}
        &\Gap(\mathbf{x}, \mathbf{y}) = D\nrm{(\vert \mathbf{Ax}-\mathbf{b}\vert - \lambda)_+}_1 + \lambda \nrm{\mathbf{x}}_1 + D\nrm{(\vert \mathbf{A}^\top\mathbf{y} \vert - \lambda)_+}_1 + \langle \mathbf{b}, \mathbf{y} \rangle + \lambda \nrm{\mathbf{y}}_1,
\end{align*}
where $|\cdot|$ and $()_+ = \max\{\cdot, 0\}$ are element-wise. We provide a brief derivation below. Since a constraint is equivalent to an indicator regularization, we move the $\ell_\infty$ constraint into the objective and denote $g_1(\cdot) = \Vert \cdot \Vert_1$, $g_2(\cdot) = \begin{cases}
    0 \text{ \quad if } \Vert \cdot \Vert_\infty \leq D \\
    \infty \text{ \ otherwise}
\end{cases}.$ By the definitions of duality gap in Definition \ref{def:saddle-obj} and convex conjugate in Definition \ref{def:conjugate}, the duality gap equals to 
\begin{align*}
    \Gap(\mathbf{x}, \mathbf{y}) &= \max_{\mathbf{y}} \lambda \{\langle \frac{1}{\lambda}(\mathbf{Ax} - \mathbf{b}), \mathbf{y} \rangle - g_1(\mathbf{y}) - g_2(\mathbf{y}) + {\nrm{\mathbf{x}}}_1 \} \\
    & \qquad - \min_{\mathbf{x}} \lambda \{\langle \frac{1}{\lambda}(\mathbf{A^\top y}), \mathbf{x} \rangle + g_1(\mathbf{x}) + g_2(\mathbf{x}) - {\nrm{\mathbf{y}}}_1 - \frac{1}{\lambda}\mathbf{b^\top y}\} \\
    &= \lambda (g_1+g_2)^\ast (\frac{1}{\lambda}(\mathbf{Ax} - \mathbf{b})) + \lambda (g_1+g_2)^\ast (\frac{1}{\lambda}(\mathbf{A^\top y})) + \lambda {\nrm{\mathbf{x}}}_1 + \lambda {\nrm{\mathbf{y}}}_1 + \mathbf{b^\top y} \\
    &= \inf_{\mathbf{u}} \{ \lambda g_1^\ast(\mathbf{u}) + \lambda g_2^\ast(\frac{1}{\lambda}(\mathbf{Ax} - \mathbf{b}) - \mathbf{u})\} + \inf_{\mathbf{v}} \{ \lambda g_1^\ast(\mathbf{v}) + \lambda g_2^\ast(\frac{1}{\lambda}(\mathbf{A^\top y}) - \mathbf{v})\} \\
    & \qquad + \lambda {\nrm{\mathbf{x}}}_1 + \lambda {\nrm{\mathbf{y}}}_1 + \mathbf{b^\top y},
\end{align*}
in which the last equality holds by Theorem 2.3.2, namely infimal convolution, in Chapter E of \citet{hiriart2004fundamentals}. By definition of the convex conjugate, the convex conjugate of a norm $g(\cdot) = \Vert \cdot \Vert_p$ is defined to be $g^\ast(\cdot) = \begin{cases}
    0 \text{ \quad if } \Vert \cdot \Vert_q \leq 1 \\
    \infty \text{ \ otherwise}
\end{cases},$ in which $\Vert \cdot \Vert_q$ is the dual norm of $\Vert \cdot \Vert_p$.  Given that $\ell_1$ and $\ell_\infty$ are dual norms to each other, $g_1^\ast(\cdot) = \begin{cases}
    0 \text{ \quad if } \Vert \cdot \Vert_\infty \leq 1 \\
    \infty \text{ \ otherwise}
\end{cases},$ $g_2^\ast (\cdot) = D\Vert \cdot \Vert_1$. Therefore the infimum is achieved when $\forall i \in [m]$, $\forall j \in [n]$,
\begin{align*}
    u_i = 
    \begin{cases}
        \frac{1}{\lambda}(\mathbf{Ax} - \mathbf{b})_i & \text{if } \vert \frac{1}{\lambda}(\mathbf{Ax} - \mathbf{b})_i \vert \leq 1 \\
        \mathrm{sgn} (\frac{1}{\lambda}(\mathbf{Ax} - \mathbf{b})_i) & \text{otherwise}
    \end{cases},
    &&
    v_j = 
    \begin{cases}
        \frac{1}{\lambda}(\mathbf{A^\top y})_j & \text{if } \vert \frac{1}{\lambda}(\mathbf{A^\top y})_j \vert \leq 1 \\
        \mathrm{sgn} (\frac{1}{\lambda}(\mathbf{A^\top y})_j) & \text{otherwise}
    \end{cases},
\end{align*}
which yields the closed-form duality gap.

\textbf{Additional Experiment Details.} We only tune the global step size $\eta^s$ and the local step size $\eta^c$. For all experiments, the parameters are searched from the combination of $\eta^s \in \{1, 3e-1, 1e-1, 3e-2, 1e-2\}$ and $\eta^c \in \{1, 3e-1, 1e-1, 3e-2, 1e-2, 3e-3, 1e-3\}$. We run each setting for 10 different random seeds and report the mean and standard deviation in Figure \ref{fig:l1_res}.

\subsection{Setup Details for Saddle Point Optimization with Low-Rank Regularization} \label{sec:exp-lowrank}

We provide additional details for the SPP with the low-rank regularization demonstrated in the main text. We start by restating its formulation below:
\begin{align*}
       &\min_{\mathbf{X} \in \mc{X}} \ \max_{\mathbf{Y} \in \mc{Y}} \mathrm{Tr} \big( (\mathbf{AX} - \mathbf{B})^\top \mathbf{Y} \big) + \lambda {\nrm{\mathbf{X}}}_\ast - \lambda \nrm{\mathbf{Y}}_\ast \\ 
        &{\mathbf{A}} \in \mathbb{R}^{n\times m},
        \qquad \mc{X} =  \{\mathbb{R}^{m \times p}: \Vert \mathbf{X} \Vert_2 \leq D\}, \\ &{\mathbf{B}} \in \mathbb{R}^{n \times p}, \qquad \  \mc{Y} =  \{\mathbb{R}^{n \times p}: \Vert \mathbf{Y} \Vert_2 \leq D\}.
\end{align*} 

\textbf{Soft-Thresholding Operator for Nuclear Norm Regularization.} By choosing the distance-generating function to be $\ell = \frac{1}{2}\Vert\mathbf{X}\Vert_\mathrm{F}^2 + \frac{1}{2}\Vert\mathbf{Y}\Vert_\mathrm{F}^2$ where $\Vert \cdot \Vert_\mathrm{F}$ denotes the Frobenius norm, the projection $\nabla \ell_{r,k}^\ast (\cdot)$ instantiates to the following element-wise singular value soft-thresholding operator \citep{cai2010singular}: 
\begin{align*}
    T_{\lambda'}(\mathbf{W}) \coloneqq  \mathbf{U} T_{\lambda'}(\mathbf{\Sigma})\mathbf{V}^\top, && T_{\lambda'}(\mathbf{\Sigma}) = \mathrm{diag}( \mathrm{sgn} (\sigma_i (\mathbf{W})) \cdot \min \{ \max \{ \sigma_i (\mathbf{W}) - \lambda', 0\}, D\}),
\end{align*}
in which $\lambda' = \lambda \eta^c (\eta^s r K + k)$, $\mathbf{W} = \mathbf{U}\mathbf{\Sigma}\mathbf{V}^\top$ is the singular value decomposition (SVD) of $\mathbf{W}$, and we overuse the notation $\sigma_i (\cdot)$ to represent the singular values.

\textbf{Closed-Form Duality Gap.} The closed-form duality gap is given by 
\begin{align*}
        \Gap(\mathbf{X}, \mathbf{Y}) &= D\nrm{\mathrm{diag} \big((\vert \sigma_i ( \mathbf{AX}-\mathbf{B} ) \vert - \lambda)_+\big)}_\ast + \lambda \nrm{\mathbf{X}}_\ast \\
        & \qquad + D\nrm{\mathrm{diag} \big((\vert \sigma_j (  \mathbf{A}^\top\mathbf{Y} ) \vert - \lambda)_+\big)}_\ast + \mathrm{Tr} \big(\mathbf{B^\top Y}\big) + \lambda \nrm{\mathbf{Y}}_\ast,
\end{align*}
We provide a brief derivation below. Since a constraint is equivalent to an indicator regularization, we move the spectral norm constraint into the objective and denote $g_1(\cdot) = \Vert \cdot \Vert_\ast$, $g_2(\cdot) = \begin{cases}
    0 \text{ \quad if } \Vert \cdot \Vert_2 \leq D \\
    \infty \text{ \ otherwise}
\end{cases}.$ By the definitions of duality gap in Definition \ref{def:saddle-obj} and convex conjugate in Definition \ref{def:conjugate}, the duality gap equals to 
\begin{align*}
    \Gap(\mathbf{X}, \mathbf{Y}) &= \max_{\mathbf{Y}} \lambda \{\mathrm{Tr} \big( \frac{1}{\lambda}(\mathbf{AX} - \mathbf{B})^\top \mathbf{Y} \big) -  g_1(\mathbf{Y}) - g_2(\mathbf{Y}) + \nrm{\mathbf{X}}_\ast \} \\
    & \qquad - \min_{\mathbf{X}} \lambda \{\{\mathrm{Tr} \big( \frac{1}{\lambda}(\mathbf{A^\top Y})^\top \mathbf{X} \big) + g_1(\mathbf{X}) + g_2(\mathbf{X}) -  {\nrm{\mathbf{Y}}}_\ast - \frac{1}{\lambda}\mathrm{Tr} \big(\mathbf{B^\top Y}\big) \} \\
    &= \lambda (g_1+g_2)^\ast (\frac{1}{\lambda}(\mathbf{AX} - \mathbf{B})) + \lambda (g_1+g_2)^\ast (\frac{1}{\lambda}(\mathbf{A^\top Y})) \\
    & \qquad + \lambda {\nrm{\mathbf{X}}}_\ast + \lambda {\nrm{\mathbf{Y}}}_\ast + \mathrm{Tr} \big(\mathbf{B^\top Y}\big) \\
    &= \inf_{\mathbf{P}} \{ \lambda g_1^\ast(\mathbf{P}) + \lambda g_2^\ast(\frac{1}{\lambda}(\mathbf{AX} - \mathbf{B}) - \mathbf{P})\} + \inf_{\mathbf{Q}} \{ \lambda g_1^\ast(\mathbf{Q}) + \lambda g_2^\ast(\frac{1}{\lambda}(\mathbf{A^\top Y}) - \mathbf{Q})\} \\
    & \qquad + \lambda {\nrm{\mathbf{X}}}_\ast + \lambda {\nrm{\mathbf{Y}}}_\ast + \mathrm{Tr} \big(\mathbf{B^\top Y}\big),
\end{align*}
in which the last equality holds by Theorem 2.3.2, namely infimal convolution, in Chapter E of \citet{hiriart2004fundamentals}. By definition of the dual norm, we know that the nuclear norm and the spectral norm are dual norms to each other. Therefore, $g_1^\ast(\cdot) = \begin{cases}
    0 \text{ \quad if } \Vert \cdot \Vert_2 \leq 1 \\
    \infty \text{ \ otherwise}
\end{cases},$ $g_2^\ast (\cdot) = D\Vert \cdot \Vert_\ast$. And the infimum is achieved when
\begin{align*}
    \sigma_i(\mathbf{P}) = 
    \begin{cases}
        \sigma_i\big(\frac{1}{\lambda}(\mathbf{Ax} - \mathbf{B})\big) & \text{if } \vert \sigma_i\big(\frac{1}{\lambda}(\mathbf{Ax} - \mathbf{B})\big) \vert \leq 1 \\
        \mathrm{sgn} \big( \sigma_i\big(\frac{1}{\lambda}(\mathbf{Ax} - \mathbf{B})\big) \big) & \text{otherwise}
    \end{cases},
    \\
    \sigma_j(\mathbf{Q}) = 
    \begin{cases}
        \sigma_j\big(\frac{1}{\lambda}(\mathbf{A^\top y})\big) & \text{if } \vert \sigma_j\big(\frac{1}{\lambda}(\mathbf{A^\top y})\big) \vert \leq 1 \\
        \mathrm{sgn} \big( \sigma_j\big(\frac{1}{\lambda}(\mathbf{A^\top y})\big) \big) & \text{otherwise}
    \end{cases},
\end{align*}
which yields the closed-form duality gap.

\textbf{Experiment Settings.} We generate a fixed pair of ${\mathbf{A}}$ and ${\mathbf{B}}$. Each entry of ${\mathbf{A}}$ and half of the columns in ${\mathbf{B}}$ follows the uniform distribution $\mc{U}_{[-1,1]}$ independently. Each entry of the variables ${\mathbf{X}}$ and ${\mathbf{Y}}$ is initialized independently from the distribution $\mc{U}_{[-1,1]}$. We take $m=600$, $n=300$, $p = 20$, $\lambda = 0.1$, $D = 0.05$. For federated learning, we simulate $M=100$ clients. For the gradient query of each client in each local update, we inject a Gaussian noise from $\mc{N}(0, \sigma^2)$.  All $M=100$ clients participate in each round; noise on each client is i.i.d. with $\sigma = 0.1$. We only tune the global step size $\eta^s$ and the local step size $\eta^c$. For all experiments, the parameters are searched from the combination of $\eta^s \in \{1, 3e-1, 1e-1, 3e-2, 1e-2\}$ and $\eta^c \in \{10, 3, 1, 3e-1, 1e-1, 3e-2, 1e-2, 3e-3, 1e-3\}$. We run each setting for 10 different random seeds and plot the mean and the standard deviation.

\section{Extended Literature Review} \label{appx:review}
\subsection{Federated Learning}
 In recent years, federated learning has received increasing attention in practice and theory. Earlier works in the field were known as ``parallel'' \citep{zinkevich2010parallelized} or ``local'' \citep{ijcai2018Convergence, stich2018local}, which are later recognized as the homogeneous case of FL where data across clients are assumed to be balanced and i.i.d. (independent and identically distributed). Generalizing with heterogeneity, federated learning was first termed in the algorithm Federated Averaging (FedAvg) \citep{mcmahan17Communication}, and it has been found appealing ever since in various applications \citep{li2020review}. On the theoretical front, \citep{stich2018local} provides the first convergence rate for FedAvg under the homogeneous setting. The rate has been improved with tighter analysis \citep{haddadpour2019local, khaled2020tighter, woodworth2020local, glasgow2022sharp} and acceleration techniques \citep{yuan2020federated, mishchenko2022proxskip}. Others also analyze FedAvg under heterogeneity \citep{haddadpour2019local, khaled2020tighter, woodworth2020minibatch} and non-i.i.d. data \citep{li2019convergence} or in light propose improvements \citep{karimireddy2020scaffold}. Recently, the idea of FL is further extended to higher-order methods \citep{bullins2021stochastic, gupta21LocalNewton, safaryan2022fednl}. Due to the page limit, we refer the readers to \citet{wang2021field} and \citet{kairouz2021advances} for more comprehensive reviews of FL. In the meantime, we point out that none of the work mentioned above covers saddle point problems or non-smooth composite or constrained problems. For distributed saddle point optimization and federated composite optimization, we defer to the following subsections.

\subsection{Saddle Point Optimization} The study of Saddle Point Optimization dates back to the very early gradient descent ascent \citep{Arrow1958Studies}. It was later improved by the important ideas of extra-gradient \citep{korpelevich1976extragradient} and optimism \citep{popov1980modification}. In light of these ideas, many algorithms were proposed for SPP \citep{Solodov1999AHA, nemirovski2004prox, nesterov2007dual, Chambolle2011AFP, mertikopoulos2018optimistic, jiang2022generalized}. Among them, in the convex-concave setting in particular, the most relevant and prominent ones are Nemirovski's mirror prox \cite{nemirovski2004prox} and Nesterov's dual extrapolation \cite{nesterov2007dual}. They generalize respectively Mirror Descent \citep{nemirovskij1983problem} and Dual Averaging \citep{nesterov2009primal} from convex optimization to monotone variational inequalities (VIs) which include SPP as one realization. Along with Tseng's Accelerated Proximal Gradient \citep{tseng2008accelerated}, they are the three methods that converge to an $\epsilon$-approximate solution in terms of duality gap at $\mc{O}(\frac{1}{T})$, the known best rate for a general convex-concave SPP \citep{ouyang2021lower, lin2020near}. Mirror prox inspired many papers \citep{antonakopoulos2019adaptive, chen2020efficient} and is later extended to the stochastic setting \citep{juditsky2011solving, mishchenko2020revisiting}, the higher-order setting \citep{bullins2022higher}, and even the composite setting \citep{he2015mirror}, whose introduction we defer to the review of composite optimization. Dual extrapolation is later extended to non-monotone VIs \citep{song2020optimistic}, yet its stochastic and composite versions are, to the best of our knowledge, not found.

From the perspective of distributed optimization, several works have made preliminary progress for smooth and unconstrained SPP in the Euclidean space. \citet{beznosikov2020distributed} investigate the distributed extra-gradient method under various conditions and provide upper and lower bounds under strongly-convex strongly-concave and non-convex non-concave assumptions. \citet{hou2021efficient} proposed FedAvg-S and SCAFFOLD-S based on FedAvg \citep{mcmahan17Communication} and SCAFFOLD \citep{karimireddy2020scaffold} for SPP, which achieves similar convergence rate to the distributed extra-gradient algorithm \citep{beznosikov2020distributed} under the strong-convexity-concavity assumption. The topic of distributed or federated saddle point optimization is also found in recent applications of interest, e.g. adversarial domain adaptation \citep{shen2023fedmm}. Yet, none of the existing works includes the study for SPP with constraints or composite possibly non-smooth regularization.

\subsection{Composite Optimization} Composite optimization has been an important topic due to its reflection of real-world complexities. Representative works include composite mirror descent \citep{duchi2010composite} and regularized dual averaging \citep{xiao2010dual, flammarion2017stochastic} that generalize mirror descent \citep{nemirovskij1983problem} and dual averaging \citep{nesterov2009primal} in the context of composite convex optimization. Composite saddle point optimization, in comparison, appears dispersedly in early-day problems in practice \citep{buades2005review, aujol2005dual}, often as a primal-dual reformulation of composite convex problems. Solving techniques such as smoothing \citep{nesterov2005smooth} and primal-dual splitting \citep{combettes2012primal} were proposed, and numerical speed-ups were studied \citep{he2015accelerating, he2016accelerated}, while systematic convergence analysis on general composite SPP came later in time \citep{he2015mirror, chambolle2016ergodic, jiang2022generalized}. Recently, \cite{tominin2021accelerated, borodich2022accelerated} also propose acceleration techniques for composite SPP.

Most related among them, the pioneering composite mirror prox (CoMP) \citep{he2015mirror} constructs auxiliary variables for the composite regularization terms as an upper bound and thus moves the non-smooth term into the problem domain. Observing that the gradient operator for the auxiliary variable is constant, CoMP operates ``as if'' there were no composite components at all \citep{he2015mirror}, and exhibits a $\mc{O}(\frac{1}{T})$ convergence rate that matches its smooth version \citep{nemirovski2004prox}. In this paper, we take a different approach that utilizes the generalized Bregman divergence and get the same rate for composite dual extrapolation.

For federated composite optimization, \cite{yuan2021federated} study Federated Mirror Descent, a natural extension of FedAvg that adapts to composite optimization under the convex setting. Along the way, they identified the ``curse of primal averaging'' specific to composite optimization in the federated learning paradigm, where the regularization imposed structure on the client models may no longer hold after server primal averaging. To resolve this issue, they further proposed Federated Dual Averaging which brings the averaging step to the dual space. On the less related constrained optimization topic,  \citet{tong2020federated} proposed a federated learning algorithm for nonconvex sparse learning under $\ell_0$ constraint. To the best of our knowledge, the field of federated learning for composite SPP remains blank, which we regard as the main focus of this paper.

\subsection{Other Tangentially Related Work} Parallel to federated learning, there is another line of work that studies \textit{decentralized optimization} or \textit{consensus optimization over networks}, in which machines communicate directly with each other based on their topological connectivity \citep{nedich2015convergence}. Classic algorithms mentioned previously are widely applied as well under this paradigm, for example, decentralized mirror descent \citep{rabbat2015multi} and decentralized (composite) dual averaging over networks \citep{duchi2011dual, liu2022decentralized}. Saddle point optimization has also been studied under this setting, including for proximal point-type methods \citep{liu2019decentralized} and extra-gradient methods \citep{rogozin2021decentralized, beznosikov2021distributed, beznosikov2022decentralized}. In particular, \citet{rogozin2021decentralized} studies decentralized ``mirror prox'' in the Euclidean space. We would like to point out that mirror prox in the Euclidean space reduces to vanilla extra-gradient methods. In addition, \citet{aybat2016primal, Xu2021Distributed} study the saddle point reformulation for composite convex objectives over decentralized networks, which essentially focus on composite convex optimization. In the general context of distributed learning of composite SPP, by the judgment of the authors, we came across no paper in decentralized optimization similar to ours. More importantly, decentralized optimization focuses on topics like time-varying network topology \citep{Kovalev2021LowerBA, kovalev2021adom} or gossip schema \citep{dimakis2006geographic}, which are fundamentally different from federated learning in terms of motivations, communication protocols, and techniques \citep{kairouz2021advances}.

For nonconvex-nonconcave saddle point problems, several federated learning methods have recently been proposed, including extra-gradient methods \citep{lee2021fast} and the Local Stochastic Gradient Descent Ascent (Local SGDA) \citep{sharma2022federated}. Yet we emphasize that our object of analysis is composite SPP with possibly non-smooth regularization, and as remarked by \citet{yuan2021federated}, non-convex optimization for composite possibly non-smooth functions is in itself intricate even for sequential optimization, involving additional assumptions and sophisticated algorithm design \citep{li2015global, bredies2015minimization}, let alone federated learning of SPP.  Thus we focus on convex-concave analysis in this paper.

\section{Additional Preliminaries, Definitions, and Remarks on Assumptions} \label{appx:additional}
In this section, we provide supplementary theoretical backgrounds for the algorithm and the convergence analysis of FeDualEx. We start by providing a more detailed introduction to the related algorithms, then list additional definitions necessary for the analysis. Before moving on to the main proof for FeDualEx, we state formally the assumptions made and provide additional remarks on the assumptions that better link them to their usage in the proof.

\subsection{Additional Preliminaries} \label{appx:prelim}
To make this paper as self-contained as possible, in this section, we provide a brief overview of mirror descent, dual averaging, and their advancement in saddle point optimization, i.e., mirror prox and dual extrapolation. More comprehensive introductions can be found in the original papers and in \citep{bubeck2015convex, CohenST21}. We slide into mirror descent from the simple and widely known projected gradient descent, namely vanilla gradient descent with constraint, therefore plus another projection of the updated sequence back to the feasible set.

\subsubsection{Mirror Descent and Dual Averaging} \label{appx:MD-DA}
We start by introducing projected gradient descent. Projected gradient descent first takes the gradient update, then projects the updated point back to the constraint by finding a feasible solution within the constraint that minimizes its Euclidean distance to the current point. The updating sequence is given below: $\forall t \in [T]$, $x_t \in \mc{X}$ whereas not necessarily for $x'_t$,
\begin{align*}
    x'_{t+1} &= x_t - \eta g(x_t) \\
    x_{t+1} &= \argmin_{x \in \mc{X}} \frac{1}{2} \nrm{x - x'_{t+1}}^2_2.
\end{align*}
\paragraph{Mirror Descent \citep{nemirovskij1983problem}.}  Mirror descent generalizes projected gradient descent to non-Euclidean space with the Bregman divergence \citep{BREGMAN1967200}. We provide the definition of the Bregman divergence below.
\begin{definition}[Bregman Divergence \citep{BREGMAN1967200}] \label{def:bregman} Let $h: \R^d \rightarrow \R \cup \{\infty\}$ be a prox function or a distance-generating function that is closed, strictly convex, and differentiable in $\mathbf{int \, dom} \, h$. The Bregman divergence for $x \in \mathbf{dom} \, h$ and $y \in \mathbf{int \, dom} \, h$ is defined to be
\begin{equation*}
    V^h_y(x) = h(x) - h(y) - \langle \nabla h(y), x-y \rangle.
\end{equation*}
\end{definition}
Mirror descent regards $\nabla h$ as a mirror map to the dual space, and follows the procedure below:
\begin{align*}
    \nabla h (x'_{t+1}) &= \nabla h(x_t) - \eta g(x_t) \\
    x_{t+1} &= \argmin_{x \in \mc{X}} V^h_{x'_{t+1}}(x).
\end{align*}
By choosing $h(\cdot) = \frac{1}{2} \Vert \cdot \Vert^2_2$ in the Euclidean space whose dual space is itself, mirror descent reduces to projected gradient descent.

Mirror descent can be presented from a proximal point of view, or in the online setting as in \citet{beck2003mirror}:
\begin{align*}
    x_{t+1} = \argmin_{x \in \mc{X}} \langle \eta g(x_t), x \rangle + V^h_{x_t}(x).
\end{align*}
 Such proximal operation with Bregman divergence is studied by others \citep{censor1992proximal}, and is recently represented by a neatly defined proximal operator \citep{CohenST21}.
\begin{definition}[Proximal Operator \citep{CohenST21}] \label{def:prox-operator} The Bregman divergence defined proximal operator is given by \begin{align*}
    \prox{}_{x'}^{h}(\cdot) \coloneqq \argmin_{x\in \mc{X}} \{ \langle \cdot, x \rangle + V^{h}_{x'}(x) \}.
\end{align*}
\end{definition}
In this spirit, the mirror descent algorithm can be written with one proximal operation: 
$$x_{t+1} = \prox{}_{x_t}^{h}(\eta g(x_t)).$$

\paragraph{Composite Mirror Descent \citep{duchi2010composite}.} Mirror descent was later generalized to composite convex functions, i.e., the ones with regularization. The key modification is to include the regularization term in the proximal operator, yet not linearize the regularization term, since it could be non-smooth and thus non-differentiable. The updating sequence is given by
\begin{align*}
    x_{t+1} = \argmin_{x \in \mc{X}} \langle \eta g(x_t), x \rangle + V^h_{x_t}(x) + \eta \psi(x).
\end{align*}
It can also be represented with a composite mirror map as in \citep{yuan2021federated}:
\begin{align*}
    x_{t+1} = \nabla (h+\eta \psi)^\ast (\nabla h (x_t) - \eta g(x_t)).
\end{align*}

\textbf{Dual Averaging \citep{nesterov2009primal}. } Compared with mirror descent, dual averaging moves the updating sequence to the dual space. The procedure of dual averaging is as follows \citep{bubeck2015convex}: 
\begin{align*}
    \nabla h (x'_{t+1}) &= \nabla h(x'_t) - \eta g(x_t) \\
    x_{t+1} &= \argmin_{x \in \mc{X}} V^h_{x'_{t+1}}(x),
\end{align*}
or equivalently as presented in \citep{nesterov2009primal} with the sequence of dual variables:
$\forall t \in [T]$, $x_t \in \mc{X}$, $\mu_t \in \mc{X}^\ast$,
\begin{align*}
    \mu_{t+1} &= \mu_t - \eta g(x_t) \\
    x_{t+1} &= \nabla h^\ast (\mu_{t+1}).
\end{align*}
This can be further simplified to 
\begin{align*}
    x_{t+1} &= \argmin_{x\in \mc{X}}\langle \eta \sum_{\tau=0}^t g(x_t), x \rangle + h(x).
\end{align*}

\paragraph{Composite Dual Averaging \citep{xiao2010dual}.} Around the same time as composite mirror descent, composite dual averaging, also known as regularized dual averaging, was proposed with a similar idea of including the regularization term in the proximal operator. As presented in the original paper \citep{xiao2010dual}:
\begin{align*}
    x_{t+1} = \argmin_{x \in \mc{X}} \langle \eta \sum_{\tau=0}^t g(x_\tau), x \rangle + \eta \beta_t h(x) + t \eta \psi(x),
\end{align*}
in which $\{\beta_t\}_{t \geq 1}$ is a non-negative and non-decreasing input sequence. \citet{flammarion2017stochastic} adopted the case with constant sequence $\beta_t = \frac{1}{\eta}$,
\begin{align*}
    x_{t+1} = \argmin_{x \in \mc{X}} \langle \eta \sum_{\tau=0}^t g(x_\tau), x \rangle + h(x) +  t \eta \psi(x),
\end{align*}
and equivalently with composite mirror map:
\begin{align*}
    \mu_{t+1} &= \mu_t - \eta g(x_t) \\
    x_{t+1} &= \nabla (h + t \eta \psi)^\ast (\mu_{t+1}),
\end{align*}
which is also presented in \citep{yuan2021federated}.

\subsubsection{Mirror Prox and Dual Extrapolation} \label{appx:MP-DE}
\paragraph{Mirror Prox \citep{nemirovski2004prox}.} Mirror prox generalizes the extra-gradient method to non-Euclidean space as mirror descent compared with projected gradient descent. It was proposed for variational inequalities (VIs), including SPP. We first present the corresponding Bregman divergence in the saddle point setting, whose definition was not included in detail in \citep{nemirovski2004prox} but was later more clearly stated in \citep{nesterov2007dual, shi2017bregman}.

\begin{definition} [Bregman Divergence for Saddle Functions \citep{nesterov2007dual}] \label{def:saddle-bregman}
Let $\ell: \mc{X} \times \mc{Y} \rightarrow \R \cup \{\infty\}$ be a distance-generating function that is closed, strictly convex, and differentiable in $\mathbf{int \, dom} \, \ell$. For $z = (x, y) \in \mc{Z} = \mc{X} \times \mc{Y}$, the function and its gradient are defined as
    \begin{align*}
        \ell(z) = h_1(x) + h_2(y), && \nabla\ell(z) = \left[ \begin{matrix} \nabla_x h_1(x) \\ \nabla_y h_2(y) \end{matrix} \right].
    \end{align*}
The Bregman divergence for $z = (x, y) \in \mathbf{dom} \, \ell$ and $z'=(x', y') \in \mathbf{int \, dom} \, \ell$ is defined to be
    \begin{align*}
        V_{z'}^\ell (z) \vcentcolon= \ell(z) - \ell(z') - \langle \nabla\ell(z'), z-z'\rangle.
    \end{align*}
Notice that our notion of $\ell$ is not a saddle function, slightly different from that in \citet{shi2017bregman}, but the Bregman divergence defined is the same as Eq. (6) in \citet{shi2017bregman} and Eq. (4.9) in \citet{nesterov2007dual}.
\end{definition}

Mirror prox can also be viewed as an extra-step mirror descent. Most intuitively, by introducing an intermediate variable $z_{t+1/2}$, its procedure is as follows:
\begin{align*}
    \nabla h (z'_{t+1/2}) &= \nabla h(z_t) - \eta g(z_t) \\
    z_{t+1/2} &= \argmin_{z \in \mc{Z}} V^h_{z'_{t+1/2}}(z) \\
    \nabla h (z'_{t+1}) &= \nabla h(z_t) - \eta g(z_{t+1/2}) \\
    z_{t+1} &= \argmin_{z \in \mc{Z}} V^h_{z'_{t+1}}(z).
\end{align*}
And it can be represented with the proximal operator in Definition \ref{def:prox-operator} as well. Following \citep{CohenST21}, $\forall t \in [T]$, $z_t, z_{t+1/2} \in \mc{Z}$,
\begin{align*}
    z_{t+1/2} &= \prox{}_{z_t}^{\ell}(\eta g(z_t)) \\
    z_{t+1} &= \prox{}_{z_t}^{\ell}(\eta g(z_{t+1/2})).
\end{align*}

\textbf{Dual Extrapolation \citep{nesterov2007dual}.} As in dual averaging, dual extrapolation moves the updating sequence of mirror prox to the dual space. Slightly different from a two-step dual averaging, dual extrapolation further initialize a fixed point in the primal space $\bar{z}$, and as presented in \citep{CohenST21}, its procedure is as follows: $\forall t \in [T]$, $z_t, z_{t+1/2} \in \mc{Z}$, $\omega_t \in \mc{Z}^\ast$,
\begin{align*}
    z_{t} &= \prox{}_{\bar{z}}^{\ell}(\omega_t) \\
    z_{t+1/2} &= \prox{}_{z_t}^{\ell}(\eta g(z_{t})) \\
    \omega_{t+1} &= \omega_t + \eta g(z_{t+1/2}).
\end{align*}
The updating sequence presented above is equivalent to that defined in the original paper \citep{nesterov2007dual}, simply replacing the $\argmax$ with $\argmin$, and the dual variables with its additive inverse in the dual space.

\subsection{Additional Definitions} \label{appx:def}

In this subsection, we list additional definitions involved in the theoretical analysis in subsequent sections.

\begin{definition}[Legendre function \citep{Rockafellar70}] A proper, convex, closed function $h:\mathbb{R}^d \rightarrow \mathbb{R} \cup \{\infty\}$ is called a Legendre function or a function of Legendre-type if  (a) $h$ is strictly convex; (b) $h$ is essentially smooth, namely h is differentiable on $\mathbf{int \ dom} \ h$, and $\nrm{\nabla h (x_t)} \rightarrow \infty$ for every sequence $\{x_t\}_{t=0}^\infty \subset \mathbf{int \ dom} \ h$ converging to a boundary point of $\mathbf{dom} \ h$ as $t \rightarrow \infty$.
\end{definition}

\begin{definition}[Convex Conjugate or Legendre–Fenchel Transformation \citep{boyd2004convex}] \label{def:conjugate} The convex conjugate of a function $h$ is defined as
\begin{align*}
    h(s) = \sup_z \{ \langle s, z \rangle - h(z) \}.
\end{align*}
\end{definition}

\begin{definition} [Differentiability of the conjugate of strictly convex function (Chapter E, Theorem 4.1.1 in \citet{hiriart2004fundamentals})] \label{def:conjugate-h} For a strictly convex function $h$, $\mathbf{int \, dom} \, h^\ast \neq \varnothing$ and $h^\ast$ is continuously differentiable on $\mathbf{int \, dom} \, h^\ast$, with gradient defined as:
    \begin{align}
        \nabla h^\ast (s) = \argmin_z \{ \langle -s, z \rangle + h(z) \}
    \end{align}
\end{definition}

\subsection{Formal Assumptions and Remarks} \label{appx:asm}
In this subsection, we state the assumptions formally and provide additional remarks that may help in understanding the theoretical analysis.
\begin{assumption}[Assumptions on the objective function] \label{asm:obj-func} 
For the composite saddle function $\phi(z) = f(x,y) + \psi_1(x) - \psi_2(y) = \frac{1}{M}\sum_{m=1}^M f_m(x,y) + \psi_1(x) - \psi_2(y)$, we assume that
    \begin{itemize}
        \item [a.] (Local Convexity of $f$) $\forall m \in [M]$, $f_m(x, y)$ is convex in $x$ and concave in $y$.
        \item [b.] (Convexity of $\psi$) $\psi_1(x)$ is convex in $x$, and $\psi_2(y)$ is convex in $y$.
    \end{itemize}
\end{assumption}

\begin{assumption}[Assumptions on the gradient operator] \label{asm:grad-op} 
For $f$ in the objective function, its gradient operator is given by $g= \left[ \begin{smallmatrix}\nabla_x f \\ -\nabla_yf\end{smallmatrix}\right]$. By the linearity of gradient operators, $g = \frac{1}{M}\sum_{m=1}^Mg_m$, and we assume that
    \begin{itemize}
        \item [a.] (Local Lipschitzness of $g$) $\forall m \in [M]$, $g_m(z) = \left[ \begin{smallmatrix}\nabla_x f_m(x, y) \\ -\nabla_yf_m(x, y)\end{smallmatrix}\right]$ is $\beta$-Lipschitz:
        \begin{align*}
            \nrm{g_m(z) - g_m(z')}_{\ast} &\leq \beta \nrm{z - z'}
        \end{align*}
        \item [b.] (Local Unbiased Estimate and Bounded Variance) For any client $m \in [M]$, the local gradient queried by some local random sample $\xi^m$ is unbiased and also bounded in variance, i.e., $\mathbb{E}_\xi [g_m(z^m;\xi^m)] = g_m(z^m)$, and
        \begin{align*}
            \mathbb{E}_\xi \big[\nrm{g_m(z^m;\xi^m) - g_m(z^m)}_\ast^2 \big] \leq \sigma^2
        \end{align*}
        \item [c.] (Bounded Gradient) $\forall m \in [M]$,
        \begin{align*}
            \nrm{g_m(z^m;\xi^m)}_\ast \leq G
        \end{align*}
    \end{itemize}
\end{assumption}

\begin{assumption} [Assumption on the distance-generating function] \label{asm:strong-convex}
    The distance-generating function $h$ is a Legendre function that is 1-strongly convex, i.e., $\forall x, y$,
    \begin{align*}
        h(y) - h(x) - \langle \nabla h(x), y-x \rangle \geq \frac{1}{2} \nrm{y-x}^2.
    \end{align*}
\end{assumption}

\begin{assumption} \label{asm:compactness}
    The domain of the optimization problem $\mc{Z}$ is compact in terms of Bregman Divergence, i.e., $\forall z, z' \in \mc{Z}$, $V_{z'}^{\ell}(z) \leq B$.
\end{assumption}

\begin{remark} \label{rmk:convexity}
    An immediate result of Assumption \ref{asm:obj-func}a is that, $\forall z = (x,y), z' = (x', y') \in \mc{Z}$
    \begin{align*}
         f(x', y') - f(x, y') &\leq \langle \nabla_x f (x', y'), x' - x\rangle, \\
        f(x', y) - f(x', y') &\leq \langle - \nabla_y f (x', y'), y' - y\rangle.
    \end{align*}
    Summing them up, 
    \begin{align*}
        f(x', y) - f(x, y') \leq \langle g(z'), z' - z \rangle.
    \end{align*}
\end{remark}

\begin{remark} \label{rmk:martingale}For any sequence of i.i.d. random variables $\xi^m_{0,0}, \xi^m_{0,1/2}, ... , \xi^m_{1,0}, \xi^m_{1,1/2},... , \xi^m_{r,k},  \xi^m_{r,k+1/2}$, let $\mc{F}_{r,k}$ denote the $\sigma$-field generated by the set $\{\xi^m_{j,t}: \forall m \in [M] \ and \ ((j = r, t \leq k) \ or \ (j < r, k \in \{0, 1/2, ... ,K-1, K-1/2\})) \}$. Then any $\xi^m_{r,k}$ is independent of $\mc{F}_{r,k-1/2}$, and Assumption \ref{asm:grad-op}b implies
\begin{align*}
    \mathbb{E}_{\mc{F}_{r,k}} \big[\nrm{g_m(z^m_{r,k};\xi^m_{r,k}) - g_m(z^m_{r,k})}_\ast^2 \mid \mc{F}_{r, k-1/2} \big] \leq \sigma^2.
\end{align*}
\end{remark}

\begin{remark} [Corollary 23.5.1. and Theorem 26.5. in \citet{Rockafellar70}] \label{rmk:bijection} For a closed convex (not necessarily differentiable) function $h$,   $\partial h$ is the inverse of $\partial h^\ast$ in the sense of multi-valued mappings, i.e., $z \in \partial h^\ast(\varsigma)$ if and only if $\varsigma \in \partial h(z)$. Furthermore, if $h$ is of Legendre-type, meaning it is essentially strictly convex and essentially smooth, then $\partial h$ yields a well-defined $\nabla h$ that acts as a bijection, i.e., $(\nabla h)^{-1} = \nabla h^\ast$.
\end{remark}

\begin{remark} \label{rmk:extension} Assumption \ref{asm:strong-convex} and Remark \ref{rmk:bijection} also trivially hold for $\ell$ from Definition \ref{def:saddle-bregman} in the saddle point setting, and eventually, the generalized distance-generating function $\ell_t$ from Definition \ref{def:generalized-bregman}. Due to the strong convexity of $\ell_t$, $\nabla\ell_t^\ast$ is well-defined as noted in Definition \ref{def:conjugate-h}. Together with the potential non-smoothness of $\ell_t$, Remark \ref{rmk:bijection} implies that $z = \nabla\ell_t^\ast(\varsigma)$ if and only if $\varsigma \in \partial\ell_t(z)$.
\end{remark}

\section{Additional Technical Lemmas} \label{appx:tech-lemmas}
In this section, we list some technical lemmas that are referenced in the proofs of the main theorem and its helping lemmas.
\setcounter{lemma}{3}
\begin{lemma}[Jensen's inequality] \label{lem:jensen}
    For a convex function $\varphi(x)$, variables $x_1, ..., x_n$ in its domain, and positive weights $a_1, ...,a_n$,
    \begin{align*}
        \varphi\Big( \frac{\sum_{i=1}^n a_i x_i}{\sum_{i=1}^n a_i} \Big) \leq \frac{\sum_{i=1}^n a_i\varphi (x_i)}{\sum_{i=1}^n a_i}, 
    \end{align*}
    and the inequality is reversed if $\varphi(x)$ is concave.
\end{lemma}

\begin{lemma}[Cauchy-Schwarz inequality \citep{strang2006linear}] \label{lem:cauchy-schwarz}
    For any $x$ and $y$ in an inner product space,
    \begin{align*}
        \langle x, y\rangle \leq \nrm{x}\nrm{y}.
    \end{align*}
\end{lemma}

\begin{lemma}[Young's inequality (Lemma 1.45. in \citet{sofonea2009variational})] \label{lem:young} Let $p, q \in \mathbb{R}$ be two conjugate exponents, that is $1 < p < \infty$, and $\frac{1}{p} + \frac{1}{q} = 1$. Then $\forall a, b \geq 0$,
\begin{align*}
    ab \leq \frac{a^p}{p} + \frac{b^q}{q}.
\end{align*}
\end{lemma}

\begin{lemma}[AM-QM inequality] \label{lem:l1-l2}
    For any set of positive integers $x_1, ... , x_n$,
    \begin{align}
        \big(\sum_{i= 1}^n x_i \big)^2 \leq n \sum_{i= 1}^n x_i^2.
    \end{align}
\end{lemma}

\begin{lemma}[Lemma 2.3 in \citet{jiang2022generalized}] \label{lem:regret} Suppose Assumption \ref{asm:obj-func} and \ref{asm:grad-op} hold, then $\forall z = (x,y)$, $z_1, ... , z_T \in \mc{Z}$ and $\theta_1, ... , \theta_T \geq 0$ with $\sum_{t=1}^T \theta_t = 1$, we have
\begin{align*}
    \phi(\sum_{t=1}^T \theta_t x_t, y) - \phi(x, \sum_{t=1}^T \theta_t y_t) \leq \sum_{t = 1}^T \theta_t [\langle g(z_t), z_t - z \rangle + \psi(z_t) - \psi(z)],
\end{align*}
in which $\psi(z) = \psi_1(x) + \psi_2(y)$.
\end{lemma}
\begin{proof} 
    For $\psi(z) = \psi_1(x) + \psi_2(y)$,
   \begin{align*}
    \phi(x_t, y) - \phi(x, y_t) &= f(x_t, y) + \psi_1(x_t) - \psi_2(y) - f(x, y_t) - \psi_1(x) + \psi_2(y_t) \\
    &= f(x_t, y) - f(x, y_t) + \psi(z_t) - \psi(z) \\
    &\leq \langle g(z_t), z_t - z \rangle + \psi(z_t) - \psi(z),
\end{align*} 
where the inequality holds by convexity-concavity of $f(x,y)$, i.e. Remark \ref{rmk:convexity}. Then sum the inequality over $t = 1, ..., T$,
\begin{align*}
    \sum_{t=1}^T \phi(\theta_t x_t, y) - \sum_{t=1}^T\phi(x,  \theta_t y_t) \leq \sum_{t=1}^T \big[\langle g(z_t), z_t - z \rangle + \psi(z_t) - \psi(z)\big].
\end{align*}
Finally, by Jensen's inequality in Lemma \ref{lem:jensen},
\begin{align*}
     \sum_{t=1}^T \phi(\theta_t x_t, y) \geq  \phi\Big(\sum_{t=1}^T\theta_t x_t, y\Big), && \sum_{t=1}^T\phi(x,  \theta_t y_t) \leq \phi\Big(x,  \sum_{t=1}^T\theta_t y_t\Big),
\end{align*}
which completes the proof. 
\end{proof}

\begin{lemma}[Theorem 4.2.1 in \citet{hiriart2004fundamentals}] \label{lem:conjugate-smooth} The conjugate of an $\alpha$-strongly convex function is $\frac{1}{\alpha}$-smooth. That is, for $h$ that is strongly convex with modulus $\alpha > 0$, $\forall x, x'$,
\begin{align*}
    \nrm{\nabla h^\ast(x) - \nabla h^\ast(x')} \leq \frac{1}{\alpha} \nrm{x - x'}.
\end{align*}
    
\end{lemma}

\begin{lemma}[Lemma 2 in \citet{flammarion2017stochastic}] \label{lem:bregman-larger} Generalized Bregman divergence upper-bounds the Bregman divergence. That is, under Assumption \ref{asm:obj-func} and \ref{asm:strong-convex}, $\forall x \in \mathbf{dom} \ h$, $\forall \mu' \in \mathbf{int \ dom} \ h_t^\ast$ where $h_t = h + t\eta\psi$,
\begin{align*}
    \tilde{V}^{h_t}_{\mu'}(x) \geq V^h_{x'}(x),
\end{align*}
in which $x' = \nabla h_t^\ast (\mu')$.
\end{lemma}

\section{Complete Analysis of FeDualEx for Composite Saddle Point Problems} \label{appx:FeDualEx-Saddle}
We begin by reformulating the updating sequences with another pair of auxiliary dual variables. Expand the prox operator  in Algorithm \ref{alg:fed-DualEx} line 6 to 8 by Definition \ref{def:generalized-prox}, and rewrite by the gradient of the conjugate function in Definition \ref{def:conjugate-h}, 
{\begin{align*}
    z^m_{r,k} &= \argmin_z\{\langle \varsigma_{r,k}^m - \bar{\varsigma}, z\rangle  + \ell_{r,k}(z)\} = \nabla \ell_{r,k}^\ast(\bar{\varsigma} - \varsigma_{r,k}^m)\\
    z^m_{r,k+1/2} &= \argmin_z\{\langle \eta^c g_m(z_{r,k}^m; \xi^m_{r,k}) - (\bar{\varsigma} - \varsigma_{r,k}^m), z \rangle + \ell_{r,k+1} (z)\} = \nabla \ell_{r,k+1}^\ast((\bar{\varsigma} - \varsigma_{r,k}^m) - \eta^c g_m(z_{r,k}^m; \xi^m_{r,k})) \\
    \varsigma^m_{r,k+1} &= \varsigma^m_{r,k} + \eta^c g_m(z^m_{r,k+1/2}; \xi^m_{r,k+1/2}) 
\end{align*}} 

Define auxiliary dual variable $\omega^m_{r,k} = \bar{\varsigma} - \varsigma_{r,k}^m$. It satisfies immediately that $z^m_{r,k} = \nabla \ell_{r,k}^\ast(\omega^m_{r,k})$, in which $\ell_{r,k}^\ast$ is the conjugate of $\ell_{r,k} = \ell + (\eta^srK+k)\eta^c\psi$. And define $\omega^m_{r,k+1/2}$ to be the dual image of the intermediate variable $z^m_{r,k+1/2}$ such that $z^m_{r,k+1/2} = \nabla \ell_{r,k+1}^\ast(\omega^m_{r,k+1/2})$. Then from the above updating sequence, we get an equivalent updating sequence for the auxiliary dual variables.
\begin{align*}
    \omega^m_{r,k+1/2} &= {\omega^m_{r,k}} - \eta g_m(z^m_{r,k}; \xi^m_{r,k})\\
    \omega^m_{r,k+1} &= \omega^m_{r,k} - \eta g_m(z^m_{r,k+1/2}; \xi^m_{r,k+1/2})
\end{align*}

Now we analyze the following shadow sequences. Define 
\begin{align*}
    \overline{\omega_{r,k}} = \frac{1}{M}\sum_{m=1}^M\omega^m_{r,k}, \qquad \qquad \overline{g_{r,k}} = \frac{1}{M}\sum_{m=1}^M g_m(z^m_{r,k}; \xi^m_{r,k}),
\end{align*}
then
\begin{align}
    \overline{\omega_{r,k+1/2}} &= \overline{\omega_{r,k}} - \eta^c \overline{g_{r,k}}, \tag{\ref{eq:seq1}} 
    \\
    \overline{\omega_{r,k+1}} &= \overline{\omega_{r,k}} - \eta^c \overline{g_{r,k+1/2}}. \tag{\ref{eq:seq2}} 
\end{align}
In the meantime, 
\begin{align} 
    \widehat{z_{r, k}} = \nabla \ell_{r,k}^\ast (\overline{\omega_{r,k}}) , \qquad \qquad \widehat{z_{r, k+1/2}} = \nabla \ell_{r,k+1}^\ast (\overline{\omega_{r,k+1/2}}). \tag{\ref{eq:def-zhat}}
\end{align}

\subsection{Main Theorem and Proof} 
\thmmain*
% \begin{theorem}
% \begin{align*}
%     \mathbb{E}\Big[\Gap\Big(\frac{1}{RK}\sum_{r=0}^{R-1}\sum_{k=0}^{K-1}\widehat{z_{r, k+1/2}}\Big)\Big] &\leq \frac{B}{\eta^cRK} + 20\beta^2(\eta^c)^3K^2G^2 + \frac{5\sigma^2\eta^c}{M} + 2^\frac{3}{2}\beta\eta^cKGB.
% \end{align*}
% Choosing step size
% \begin{align*}
%     \eta^c = \min \{\frac{1}{5\beta^2}, \frac{B^\frac{1}{4}}{20^\frac{1}{4}\beta^\frac{1}{2}G^\frac{1}{2}K^\frac{3}{4}R^\frac{1}{4}}, \frac{B^\frac{1}{2}M^\frac{1}{2}}{5^\frac{1}{2}\sigma R^\frac{1}{2} K^\frac{1}{2}}, \frac{1}{2^\frac{3}{4}\beta^\frac{1}{2}G^\frac{1}{2}KR^\frac{1}{2}}\},
% \end{align*}
% \begin{align*}
%     \mathbb{E}\Big[\Gap\Big(\frac{1}{RK}\sum_{r=0}^{R-1}\sum_{k=0}^{K-1}\widehat{z_{r, k+1/2}}\Big)\Big] &\leq \frac{5\beta^2B}{RK} + \frac{20^\frac{1}{4}\beta^\frac{1}{2}G^\frac{1}{2}B^\frac{3}{4}}{K^\frac{1}{4}R^\frac{3}{4}} + \frac{5^\frac{1}{2}\sigma B^\frac{1}{2}}{M^\frac{1}{2}R^\frac{1}{2} K^\frac{1}{2}} + \frac{2^\frac{3}{4}\beta^\frac{1}{2}G^\frac{1}{2}B}{R^\frac{1}{2}}.
% \end{align*}
% \end{theorem}
\begin{proof}
The proof of the main theorem relies on Lemma \ref{lem:non-smooth}, the bound for the non-smooth term, and Lemma \ref{lem:smooth}, the bound for the smooth term. These two lemmas are combined in Lemma \ref{lem:main-grad} and then yield the per-step progress for FeDualEx. The three lemmas are listed and proved right after this theorem. Here, we finish proving the main theorem from the per-step progress. 

Starting from Lemma \ref{lem:main-grad}, we telescope for all local updates $k \in \{0, ..., K-1\}$ after the same communication round $r$. 
{\begin{align*}
    \eta^c\mathbb{E} & \Big[\sum_{k=0}^{K-1} \big[\langle g(\widehat{z_{r, k+1/2}}), \widehat{z_{r, k+1/2}} - z\rangle + \psi(\widehat{z_{r, k+1/2}}) - \psi(z)\big] \Big] \\
    &\leq \tilde{V}_{\overline{\omega_{r,0}}}^{\ell_{r,0}}(z) -\tilde{V}_{\overline{\omega_{r,K}}}^{\ell_{r,K}}(z) + \frac{5\sigma^2(\eta^c)^2K}{M} + 20 \sum_{k=0}^{K-1} \beta^2(\eta^c)^4(k+1)^2G^2 + 2^\frac{3}{2} \sum_{k=0}^{K-1} \beta (\eta^c)^2(k+1)GB \\
    &\leq \tilde{V}_{\overline{\omega_{r,0}}}^{\ell_{r,0}}(z) -\tilde{V}_{\overline{\omega_{r,K}}}^{\ell_{r,K}}(z) + \frac{5\sigma^2(\eta^c)^2K}{M} + 20 \sum_{k=0}^{K-1} \beta^2(\eta^c)^4K^2G^2 + 2^\frac{3}{2} \sum_{k=0}^{K-1} \beta (\eta^c)^2 K GB \\
    &\leq \tilde{V}_{\overline{\omega_{r,0}}}^{\ell_{r,0}}(z) -\tilde{V}_{\overline{\omega_{r,K}}}^{\ell_{r,K}}(z) + \frac{5\sigma^2(\eta^c)^2K}{M} + 20\beta^2(\eta^c)^4K^3G^2 + 2^\frac{3}{2}\beta(\eta^c)^2K^2GB.
\end{align*}}
As we initialize the local dual updates on all clients after each communication with the dual average of the previous round's last update, $\forall r \in \{1, ..., R\}$, the first variable in this round $\overline{\omega_{r, 0}}$ is the same as the last variable $\overline{\omega_{r-1, 0}}$ in the previous round. As a result, taking the server step size $\eta^s = 1$, we can further telescope across all rounds and have
\begin{align*}
\eta^c\mathbb{E} & \Big[\sum_{r=0}^{R-1}\sum_{k=0}^{K-1} \big[\langle g(\widehat{z_{r, k+1/2}}), \widehat{z_{r, k+1/2}} - z\rangle + \psi(\widehat{z_{r, k+1/2}}) - \psi(z)\big] \Big] \\
&\leq \tilde{V}_{\overline{\omega_{0,0}}}^{\ell_{0,0}}(z) -\tilde{V}_{\overline{\omega_{R,K}}}^{\ell_{R,K}}(z) + \frac{5\sigma^2(\eta^c)^2KR}{M} + 20\beta^2(\eta^c)^4K^3RG^2 + 2^\frac{3}{2}\beta(\eta^c)^2K^2RGB .
\end{align*}
Notice that the generalized Bregman divergence $\tilde{V}_{\overline{\omega_{0,0}}}^{\ell_{0,0}}(z) = \tilde{V}_{\bar{\varsigma}-\varsigma_0}^{\ell_{0,0}}(z) = \tilde{V}_{\bar{\varsigma}}^{\ell}(z) = V_{z_0}^{\ell}(z)$, where $z_0 = \nabla \ell^\ast(\bar{\varsigma})$. Thus, by Assumption \ref{asm:compactness}, $\tilde{V}_{\overline{\omega_{0,0}}}^{\ell_{0,0}}(z) \leq B$. Dividing $\eta^cKR$ on both sides of the equation, we get
\begin{align*}
    \eta^c\mathbb{E} & \Big[\frac{1}{RK}\sum_{r=0}^{R-1}\sum_{k=0}^{K-1} \big[\langle g(\widehat{z_{r, k+1/2}}), \widehat{z_{r, k+1/2}} - z\rangle + \psi(\widehat{z_{r, k+1/2}}) - \psi(z)\big] \Big] \\
    &\leq \frac{B}{\eta^cRK} + \frac{5\sigma^2\eta^c}{M} + 20\beta^2(\eta^c)^3K^2G^2 + 2^\frac{3}{2}\beta\eta^cKGB.
\end{align*}
Finally, applying Lemma \ref{lem:regret} completes the proof. 
\end{proof}

% \begin{lemma} [Bounding the Regularization Term] %\label{lem:non-smooth} 
% $\forall z$, 
% \begin{align*}
%  \eta^c \big[ \psi(\widehat{z_{r, k+1/2}}) - \psi(z)\big] &= \tilde{V}_{\overline{\omega_{r,k}}}^{\ell_{r,k}}(z)  - \tilde{V}_{\overline{\omega_{r,k+1}}}^{\ell_{r,k+1}}(z) - \tilde{V}_{\overline{\omega_{r,k}}}^{\ell_{r,k}}(\widehat{z_{r, k+1/2}}) - \tilde{V}_{\overline{\omega_{r,k+1/2}}}^{\ell_{r,k+1}}(\widehat{z_{r, k+1}}) \\ 
%     & \quad + \eta^c\langle \overline{g_{r,k+1/2}} - \overline{g_{r,k}}, \widehat{z_{r, k+1/2}} - \widehat{z_{r, k+1}} \rangle + \eta^c\langle \overline{g_{r,k+1/2}}, z - \widehat{z_{r, k+1/2}} \rangle
% \end{align*}
% \end{lemma}
\lemnonsmooth*
\begin{proof}
    By the definition of generalized Bregman divergence and the updating sequence in Eq. \eqref{eq:seq1}, $\forall z$,
    \begin{align*}
        \tilde{V}_{\overline{\omega_{r,k+1/2}}}^{\ell_{r,k+1}}(z) &= \ell_{r, k+1}(z) - \ell_{r, k+1}(\widehat{z_{r,k+1/2}}) - \langle\overline{\omega_{r,k+1/2}}, z - \widehat{z_{r,k+1/2}}\rangle \\
        &= \ell_{r, k+1}(z) - \ell_{r, k+1}(\widehat{z_{r,k+1/2}}) - \langle\overline{\omega_{r,k}} - \eta^c \overline{g_{r,k}}, z - \widehat{z_{r,k+1/2}}\rangle \\
        &= \ell_{r,k}(z) - \ell_{r,k}(\widehat{z_{r,k+1/2}}) + \eta^c \big[ \psi(z) - \psi(\widehat{z_{r,k+1/2}})\big] \\
        &\qquad - \langle\overline{\omega_{r,k}}, z - \widehat{z_{r,k+1/2}}\rangle + \eta^c \langle\overline{g_{r,k}}, z - \widehat{z_{r,k+1/2}}\rangle. \numberthis \label{eq:lem-nonsmooth1}
    \end{align*}
    Similarly, we can have for the updating sequence in Eq. \eqref{eq:seq2} that $\forall z$,
    \begin{align}
        \tilde{V}_{\overline{\omega_{r,k+1}}}^{\ell_{r,k+1}}(z)
        &= \ell_{r,k}(z) - \ell_{r,k}(\widehat{z_{r,k+1}}) + \eta^c \big[ \psi(z) - \psi(\widehat{z_{r,k+1}})\big] - \langle\overline{\omega_{r,k}}, z - \widehat{z_{r,k+1}}\rangle + \eta^c \langle\overline{g_{r,k+1/2}}, z - \widehat{z_{r,k+1}}\rangle. \label{eq:lem-nonsmooth2}
    \end{align}
Plug $z = \widehat{z_{r,k+1}}$ into Eq. \eqref{eq:lem-nonsmooth1},
    \begin{align*}
        \tilde{V}_{\overline{\omega_{r,k+1/2}}}^{\ell_{r,k+1}}(\widehat{z_{r,k+1}}) &= \ell_{r,k}(\widehat{z_{r,k+1}}) - \ell_{r,k}(\widehat{z_{r,k+1/2}}) + \eta^c \big[ \psi(\widehat{z_{r,k+1}}) - \psi(\widehat{z_{r,k+1/2}})\big] \\
        &\quad - \langle\overline{\omega_{r,k}}, \widehat{z_{r,k+1}} - \widehat{z_{r,k+1/2}}\rangle + \eta^c \langle\overline{g_{r,k}}, \widehat{z_{r,k+1}} - \widehat{z_{r,k+1/2}}\rangle.
    \end{align*}
Add this up with Eq. \eqref{eq:lem-nonsmooth2},
\begin{align*}
        \tilde{V}_{\overline{\omega_{r,k+1/2}}}^{\ell_{r,k+1}}(\widehat{z_{r,k+1}}) + \tilde{V}_{\overline{\omega_{r,k+1}}}^{\ell_{r,k+1}}(z) &= \underbrace{\ell_{r,k}(z) - \ell_{r,k}(\widehat{z_{r,k+1/2}}) - \langle\overline{\omega_{r,k}}, z - \widehat{z_{r,k+1/2}}\rangle}_{A1} + \eta^c \big[ \psi(z) - \psi(\widehat{z_{r,k+1/2}})\big] \\
        &\qquad + \underbrace{\eta^c \langle\overline{g_{r,k}}, \widehat{z_{r,k+1}} - \widehat{z_{r,k+1/2}}\rangle + \eta^c \langle\overline{g_{r,k+1/2}}, z - \widehat{z_{r,k+1}}\rangle}_{A2}.
    \end{align*}
For $A1$ we have
{\begin{align*}
    A1 &= \ell_{r,k}(z) - \ell_{r,k}(\widehat{z_{r,k}}) - \langle\overline{\omega_{r,k}}, z - \widehat{z_{r,k}}\rangle - \ell_{r,k}(\widehat{z_{r,k+1/2}}) + \ell_{r,k}(\widehat{z_{r,k}}) + \langle\overline{\omega_{r,k}}, \widehat{z_{r,k+1/2}} - \widehat{z_{r,k}}\rangle \\
    &= \tilde{V}_{\overline{\omega_{r,k}}}^{\ell_{r,k}}(z) - \tilde{V}_{\overline{\omega_{r,k}}}^{\ell_{r,k}}(\widehat{z_{r, k+1/2}}).
\end{align*}}
For $A2$ we have
\begin{align*}
    A2 &= \eta^c \langle\overline{g_{r,k}}, \widehat{z_{r,k+1}} - \widehat{z_{r,k+1/2}}\rangle + \eta^c \langle\overline{g_{r,k+1/2}}, \widehat{z_{r,k+1/2}} - \widehat{z_{r,k+1}}\rangle + \eta^c \langle\overline{g_{r,k+1/2}}, z - \widehat{z_{r,k+1/2}}\rangle\\
    &= \eta^c \langle\overline{g_{r,k+1/2}}, z - \widehat{z_{r,k+1/2}}\rangle + \eta^c\langle\overline{g_{r,k+1/2}} - \overline{g_{r,k}}, \widehat{z_{r,k+1/2}} - \widehat{z_{r,k+1}}\rangle
\end{align*}
Plug $A1$ and $A2$ back in completes the proof. 
\end{proof}

For the purpose of clarity, we demonstrate how we generate the terms to be separately bounded for the smooth part with the following Lemma \ref{lem:smooth}, which holds trivially by the linearity of the gradient operator $g = \frac{1}{M}\sum_{m=1}^Mg_m$ and then direct cancellation.
\lemsmooth*
% \begin{lemma} [Bounding the Smooth Term] % \label{lem:smooth}
% $\forall z$, 
% \begin{align*}
%      \langle g(\widehat{z_{r, k+1/2}}), \widehat{z_{r, k+1/2}} - z\rangle &= \langle \overline{g_{r,k+1/2}}, \widehat{z_{r, k+1/2}} - z\rangle + \langle \frac{1}{M} \sum_{m=1}^M g_m(z_{r,k+1/2}^m) - \overline{g_{r,k+1/2}}, \widehat{z_{r, k+1/2}} - z\rangle \\ 
%     &\quad + \langle \frac{1}{M} \sum_{m=1}^M [g_m(\widehat{z_{r, k+1/2}}) - g_m(z_{r,k+1/2}^m)], \widehat{z_{r, k+1/2}} - z\rangle
% \end{align*}
% \end{lemma}

Based on the previous two lemmas, we arrive at the following lemma that bounds the per-step progress of FeDualEx.

\setcounter{lemma}{2}
\begin{lemma} [Per-step Progress for FeDualEx in Saddle Point Setting] \label{lem:main-grad} For $\eta^c \leq \frac{1}{5\beta^2}$,
\begin{align*}
    \eta^c\mathbb{E} & \big[\langle g(\widehat{z_{r, k+1/2}}), \widehat{z_{r, k+1/2}} - z\rangle + \psi(\widehat{z_{r, k+1/2}}) - \psi(z) \big] \\ &\leq \tilde{V}_{\overline{\omega_{r,k}}}^{\ell_{r,k}}(z) -\tilde{V}_{\overline{\omega_{r,k+1}}}^{\ell_{r,k+1}}(z) + \frac{5\sigma^2(\eta^c)^2}{M} + 20\beta^2(\eta^c)^4(k+1)^2G^2 + 2^\frac{3}{2} \beta (\eta^c)^2(k+1)GB.
\end{align*}
\end{lemma}
\begin{proof}
Based on the previous two lemmas, we can get the following simply by summing them up, in which we denote the left-hand side as $\mathrm{LHS}$ for simplicity.
\begin{align*}
    \mathrm{LHS} &\coloneqq \eta^c\big[\langle g(\widehat{z_{r, k+1/2}}), \widehat{z_{r, k+1/2}} - z\rangle + \psi(\widehat{z_{r, k+1/2}}) - \psi(z) \big] \\ &\leq \tilde{V}_{\overline{\omega_{r,k}}}^{\ell_{r,k}}(z) - \tilde{V}_{\overline{\omega_{r,k+1}}}^{\ell_{r,k+1}}(z) \underbrace{- \tilde{V}_{\overline{\omega_{r,k}}}^{\ell_{r,k}}(\widehat{z_{r, k+1/2}}) - \tilde{V}_{\overline{\omega_{r,k+1/2}}}^{\ell_{r,k+1}}(\widehat{z_{r, k+1}})}_{A3}  \\
    &\qquad + \eta^c\langle \overline{g_{r,k+1/2}} - \overline{g_{r,k}}, \widehat{z_{r, k+1/2}} - \widehat{z_{r, k+1}} \rangle  + \eta^c\langle \frac{1}{M} \sum_{m=1}^M g_m(z_{r,k+1/2}^m) - \overline{g_{r,k+1/2}}, \widehat{z_{r, k+1/2}} - z\rangle \\ 
    &\qquad + \eta^c\langle \frac{1}{M} \sum_{m=1}^M [g_m(\widehat{z_{r, k+1/2}}) - g_m(z_{r,k+1/2}^m)], \widehat{z_{r, k+1/2}} - z\rangle
\end{align*}
For the two generalized Bregman divergence terms in $A3$, we bound them by Lemma \ref{lem:bregman-larger} and the strong convexity of $\ell$ in Remark \ref{rmk:extension}, 
\begin{align*}
    A3 &\leq - V_{\widehat{z_{r,k}}}^{\ell}(\widehat{z_{r, k+1/2}}) - V_{\widehat{z_{r,k+1/2}}}^{\ell}(\widehat{z_{r, k+1}}) \\
    &\leq - \frac{1}{2} \nrm{\widehat{z_{r,k}} - \widehat{z_{r, k+1/2}}}^2 - \frac{1}{2} \nrm{\widehat{z_{r,k+1/2}} - \widehat{z_{r, k+1}}}^2
\end{align*}
As a result,
\begin{align*}
    \mathrm{LHS} &\leq \tilde{V}_{\overline{\omega_{r,k}}}^{\ell_{r,k}}(z) - \tilde{V}_{\overline{\omega_{r,k+1}}}^{\ell_{r,k+1}}(z) - \frac{1}{2} \nrm{\widehat{z_{r,k}} - \widehat{z_{r, k+1/2}}}^2   \\
    & \quad \underbrace{ - \frac{1}{2} \nrm{\widehat{z_{r,k+1/2}} - \widehat{z_{r, k+1}}}^2 + \eta^c\langle \overline{g_{r,k+1/2}} - \overline{g_{r,k}}, \widehat{z_{r, k+1/2}} - \widehat{z_{r, k+1}} \rangle}_{A4} \\
    &\quad + \eta^c\langle \frac{1}{M} \sum_{m=1}^M g_m(z_{r,k+1/2}^m) - \overline{g_{r,k+1/2}}, \widehat{z_{r, k+1/2}} - z\rangle \\ 
    &\quad + \eta^c\langle \frac{1}{M} \sum_{m=1}^M [g_m(\widehat{z_{r, k+1/2}}) - g_m(z_{r,k+1/2}^m)], \widehat{z_{r, k+1/2}} - z\rangle.
\end{align*}
$A4$ can be bounded with Cauchy-Schwarz (Lemma \ref{lem:cauchy-schwarz}) inequality and Young's inequality (Lemma \ref{lem:young}).
\begin{align*}
    A4 &\leq - \frac{1}{2}\nrm{\widehat{z_{r,k+1/2}} - \widehat{z_{r,k+1}}}^2 + \eta^c\nrm{\overline{g_{r,k+1/2}} - \overline{g_{r,k}}}_\ast\nrm{\widehat{z_{r, k+1/2}} - \widehat{z_{r, k+1}}} \\
    &\leq - \frac{1}{2}\nrm{\widehat{z_{r,k+1/2}} - \widehat{z_{r,k+1}}}^2 + \frac{(\eta^c)^2}{2}\nrm{\overline{g_{r,k+1/2}} - \overline{g_{r,k}}}_\ast^2 + \frac{1}{2}\nrm{\widehat{z_{r, k+1/2}} - \widehat{z_{r, k+1}}}^2 \\
    &= \frac{(\eta^c)^2}{2}\nrm{\overline{g_{r,k+1/2}} - \overline{g_{r,k}}}_\ast^2.
\end{align*}
Then we have
\begin{align*}
    \eta^c\big(\phi(\widehat{z_{r, k+1/2}}) - \phi(z) \big) &\leq \tilde{V}_{\overline{\omega_{r,k}}}^{\ell_{r,k}}(z) -\tilde{V}_{\overline{\omega_{r,k+1}}}^{\ell_{r,k+1}}(z) - \frac{1}{2} \nrm{\widehat{z_{r,k}} - \widehat{z_{r, k+1/2}}}^2 + \frac{(\eta^c)^2}{2}\nrm{\overline{g_{r,k+1/2}} - \overline{g_{r,k}}}_\ast^2 \\
    &\quad + \eta^c\langle \frac{1}{M} \sum_{m=1}^M g_m(z_{r,k+1/2}^m) - \overline{g_{r,k+1/2}}, \widehat{z_{r, k+1/2}} - z\rangle \\ 
    &\quad + \eta^c\langle \frac{1}{M} \sum_{m=1}^M [g_m(\widehat{z_{r, k+1/2}}) - g_m(z_{r,k+1/2}^m)], \widehat{z_{r, k+1/2}} - z\rangle.
\end{align*}
Taking expectations on both sides we get
\begin{align*}
    \eta^c\mathbb{E}\big[\phi(\widehat{z_{r, k+1/2}}) - \phi(z) \big] &\leq \tilde{V}_{\overline{\omega_{r,k}}}^{\ell_{r,k}}(z) -\tilde{V}_{\overline{\omega_{r,k+1}}}^{\ell_{r,k+1}}(z) \underbrace{- \frac{1}{2} \mathbb{E}\big[\nrm{\widehat{z_{r,k}} - \widehat{z_{r, k+1/2}}}^2\big]}_{B1} + \underbrace{\frac{(\eta^c)^2}{2}\mathbb{E}\big[\nrm{\overline{g_{r,k+1/2}} - \overline{g_{r,k}}}_\ast^2\big]}_{B2} \\
    &\qquad + \underbrace{\eta^c\mathbb{E}\big[\langle \frac{1}{M} \sum_{m=1}^M g_m(z_{r,k+1/2}^m) - \overline{g_{r,k+1/2}}, \widehat{z_{r, k+1/2}} - z\rangle\big]}_{B3} \\
    &\qquad + \underbrace{\eta^c\mathbb{E}\big[\langle \frac{1}{M} \sum_{m=1}^M [g_m(\widehat{z_{r, k+1/2}}) - g_m(z_{r,k+1/2}^m)], \widehat{z_{r, k+1/2}} - z\rangle\big]}_{B4}.
\end{align*}
B2 is bounded in Lemma \ref{lem:gradient-diff}. Therefore, we have
\begin{align*}
    B1 + B2 &\leq \frac{(\eta^c)^2}{2} \big(\frac{10\sigma^2}{M} + 40\beta^2(\eta^c)^2(k+1)^2G^2 \big) + \frac{5\eta^c\beta^2}{2} \mathbb{E}\Big[\nrm{ \widehat{z_{r, k+1/2}} - \widehat{z_{r, k}}}^2 \Big] - \frac{1}{2} \mathbb{E}\big[\nrm{\widehat{z_{r,k}} - \widehat{z_{r, k+1/2}}}^2\big] \\
    &= \frac{(\eta^c)^2}{2} \big(\frac{10\sigma^2}{M} + 40\beta^2(\eta^c)^2(k+1)^2G^2 \big) + \frac{5\eta^c\beta^2 - 1}{2} \mathbb{E}\Big[\nrm{ \widehat{z_{r, k+1/2}} - \widehat{z_{r, k}}}^2 \Big] \\
    &\leq \frac{5\sigma^2(\eta^c)^2}{M} + 20\beta^2(\eta^c)^4(k+1)^2G^2,
\end{align*}
for $\eta^c \leq \frac{1}{5\beta^2}$.

B3 is zero after taking the expectation as shown in Lemma \ref{lem:unbiased-grad}. B4 is bounded in Lemma \ref{lem:gradient-inner}. Plugging the bounds for $B1+B2, \ B3$, and $B4$ back in completes the proof.
\end{proof}

\subsection{Helping Lemmas}
In this section, we list the helping lemmas that were referenced in the proof of Lemma \ref{lem:non-smooth}, \ref{lem:smooth}, and \ref{lem:main-grad}. 
\setcounter{lemma}{10}
\begin{lemma} [Unbiased Gradient Estimate] \label{lem:unbiased-grad} Under Assumption \ref{asm:obj-func} and \ref{asm:grad-op},
    \begin{align*}
        \eta^c\mathbb{E}_{\mc{F}_{r,k+1/2}} \big[ \langle \frac{1}{M} \sum_{m=1}^M g_m(z_{r,k+1/2}^m) - \overline{g_{r,k+1/2}}, \widehat{z_{r, k+1/2}} - z\rangle\big] = 0
    \end{align*}
\end{lemma}
\begin{proof} By the unbiased gradient estimate in Assumption \ref{asm:grad-op}b and its following Remark \ref{rmk:martingale},
    \begin{align*}
        \eta^c\mathbb{E}_{\mc{F}_{r,k+1/2}} & \big[ \langle \frac{1}{M} \sum_{m=1}^M g_m(z_{r,k+1/2}^m) - \overline{g_{r,k+1/2}}, \widehat{z_{r, k+1/2}} - z\rangle\big] \\
        &= \eta^c \mathbb{E}_{\mc{F}_{r,k}} \big[\mathbb{E}_{\mc{F}_{r,k+1/2}} \big[ \langle \frac{1}{M} \sum_{m=1}^M g_m(z_{r,k+1/2}^m) - \overline{g_{r,k+1/2}}, \widehat{z_{r, k+1/2}} - z\rangle \big\vert \mc{F}_{r,k}\big]\big] \\
        &= 0.
    \end{align*}
\end{proof}

\begin{lemma} [Bounded Client Drift under Assumption \ref{asm:grad-op}c] \label{lem:client-server-z}
$\forall m \in [M]$, $\forall k \in \{0, ..., K-1\}$,
    \begin{align*}
       \nrm{\widehat{z_{r, k+1/2}} - z_{r,k+1/2}^m} &\leq 2 \eta^c (k+1) G \\
        \nrm{\widehat{z_{r, k}} - z_{r,k}^m} &\leq 2 \eta^c k G
    \end{align*}
\end{lemma}
\begin{proof}
    By the smoothness of the conjugate of a strongly convex function, i.e., Lemma \ref{lem:conjugate-smooth},
    \begin{align*}
        \nrm{\widehat{z_{r, k+1/2}} - z_{r,k+1/2}^m} &= \nrm{\nabla \ell_{r,k}^\ast(\overline{\omega_{r, k+1/2}}) - \nabla \ell_{r,k}^\ast (\omega_{r,k+1/2}^m)} \\
        &\leq \nrm{\overline{\omega_{r, k+1/2}} - \omega_{r,k+1/2}^m}_\ast
    \end{align*}
    After the same round of communication, by the updating sequence, we have $\forall m \in [M]$:
    \begin{align*}
        \omega_{r,k+1/2}^m &= \omega_{r,k}^m - \eta^c g_m(z^m_{r,k}; \xi^m_{r,k}) \\
        &= - \eta^c \sum_{\ell = 0}^{k-1} g_m(z^m_{r,\ell+1/2}; \xi^m_{r,\ell+1/2}) - \eta^c g_m(z^m_{r,k}; \xi^m_{r,k})
    \end{align*}
    Immediately after each round of communication, all machines are synchronized, i.e., $\forall m_1, m_2 \in [M]$, $\omega_{r,0}^{m_1} = \omega_{r,0}^{m_2}$. Therefore, $\forall k \in \{0, ..., K-1\}$,
    \begin{align*}
        \omega_{r,k+1/2}^{m_1} - \omega_{r,k+1/2}^{m_2} &= - \eta^c \sum_{\ell = 0}^{k-1} g_{m_1}(z^{m_1}_{r,\ell+1/2}; \xi^{m_1}_{r,\ell+1/2}) - \eta^c g_{m_1}(z^{m_1}_{r,k}; \xi^{m_1}_{r,k}) \\
        & \qquad + \eta^c \sum_{\ell = 0}^{k-1} g_{m_2}(z^{m_2}_{r,\ell+1/2}; \xi^{m_2}_{r,\ell+1/2}) + \eta^c g_{m_2}(z^{m_2}_{r,k}; \xi^{m_2}_{r,k}) 
    \end{align*}
    Then $\forall m_1, m_2 \in [M]$, $\forall k \in \{0, ..., K-1\}$, by triangle inequality, Jensen's inequality, and the bounded gradient Assumption \ref{asm:grad-op}c,
    \begin{align*}
        \nrm{\omega_{r,k+1/2}^{m_1} - \omega_{r,k+1/2}^{m_2}}_\ast &\leq \eta^c\big(\sum_{\ell = 0}^{k-1}\nrm{ g_{m_1}(z^{m_1}_{r,\ell+1/2}; \xi^{m_1}_{r,\ell+1/2})}_\ast + \nrm{g_{m_1}(z^{m_1}_{r,k}; \xi^{m_1}_{r,k})}_\ast \\
        & \qquad + \sum_{\ell = 0}^{k-1}\nrm{ g_{m_2}(z^{m_2}_{r,\ell+1/2}; \xi^{m_2}_{r,\ell+1/2})}_\ast + \nrm{g_{m_2}(z^{m_2}_{r,k}; \xi^{m_2}_{r,k}) }_\ast \big) \\
        &\leq 2\eta^c (k+1) G.
    \end{align*}
    As a result,
    \begin{align*}
        \nrm{\widehat{z_{r, k+1/2}} - z_{r,k+1/2}^m} &\leq \nrm{\overline{\omega_{r, k+1/2}} - \omega_{r,k+1/2}^m}_\ast \\
        &\leq \sup_{m_1, m_2}\nrm{\omega_{r,k+1/2}^{m_1} - \omega_{r,k+1/2}^{m_2}}_\ast  \\
        &\leq 2\eta^c (k+1) G.
    \end{align*}
    Similarly, we can show that 
    \begin{align*}
        \nrm{\widehat{z_{r, k}} - z_{r,k}^m} &\leq 2 \eta^c k G.
    \end{align*}
\end{proof}

\begin{lemma} \label{lem:gradient-inner}  Under Assumption \ref{asm:obj-func}-\ref{asm:compactness},
    \begin{align*}
        \eta^c\mathbb{E}\big[\langle \frac{1}{M} \sum_{m=1}^M [g_m(\widehat{z_{r, k+1/2}}) - g_m(z_{r,k+1/2}^m)], \widehat{z_{r, k+1/2}} - z\rangle\big] \leq 2^\frac{3}{2} \beta(\eta^c)^2(k+1)GB.
    \end{align*}
\end{lemma}
\begin{proof} The proof of this lemma relies on the bounded client drift in Lemma \ref{lem:client-server-z}. We start by splitting the inner product using Cauchy-Schwarz inequality in Lemma \ref{lem:cauchy-schwarz}, and state the reference for the following derivation in the parenthesis. 
    \begin{align*}
        \eta^c\mathbb{E} & \big[\langle \frac{1}{M} \sum_{m=1}^M [g_m(\widehat{z_{r, k+1/2}}) - g_m(z_{r,k+1/2}^m)], \widehat{z_{r, k+1/2}} - z\rangle\big] \\
        & \leq \eta^c\mathbb{E} \big[ \nrm{\frac{1}{M} \sum_{m=1}^M [g_m(\widehat{z_{r, k+1/2}}) - g_m(z_{r,k+1/2}^m)]}_\ast \nrm{ \widehat{z_{r, k+1/2}} - z} \big] \\
        & \leq \eta^c\mathbb{E} \big[ \frac{1}{M} \sum_{m=1}^M \nrm{ g_m(\widehat{z_{r, k+1/2}}) - g_m(z_{r,k+1/2}^m)}_\ast \nrm{ \widehat{z_{r, k+1/2}} - z} \big] & \text{(Jensen's)} \\
        & \leq \eta^c\mathbb{E} \big[ \frac{1}{M} \sum_{m=1}^M \beta\nrm{ \widehat{z_{r, k+1/2}} - z_{r,k+1/2}^m}_\ast \nrm{ \widehat{z_{r, k+1/2}} - z} \big] & \text{(Smoothness)} \\
        & \leq \eta^c\mathbb{E} \big[ \frac{1}{M} \sum_{m=1}^M 2\beta \eta^c(k+1)G \nrm{ \widehat{z_{r, k+1/2}} - z} \big] & \text{(Lemma \ref{lem:client-server-z})} \\
        & \leq \eta^c\mathbb{E} \big[ 2\beta \eta^c(k+1)G \cdot \sqrt{2} V^\ell_{z}(\widehat{z_{r, k+1/2}}) \big] & \text{(Strong-convexity of $\ell$)} \\
        &\leq 2^\frac{3}{2} \beta (\eta^c)^2(k+1)GB & \text{(Assumption \ref{asm:compactness})}
    \end{align*}
\end{proof}

\begin{lemma} [Difference of Gradient and Extra-gradient] \label{lem:gradient-diff} Under Assumption \ref{asm:obj-func}-\ref{asm:compactness},
\begin{align*}
\mathbb{E} \big[ \nrm{\overline{g_{r,k+1/2}} - \overline{g_{r,k}}}_\ast^2 \big] \leq \frac{10\sigma^2}{M} + 40\beta^2(\eta^c)^2(k+1)^2G^2 + 5 \beta^2 \mathbb{E}\Big[\nrm{ \widehat{z_{r, k+1/2}} - \widehat{z_{r, k}}}^2 \Big].
\end{align*}
\end{lemma}
\begin{proof}
By Lemma \ref{lem:l1-l2},
\begin{align*}
    &\mathbb{E}_{\mc{F}_{r,k+1/2}}  \big[\nrm{\overline{g_{r,k+1/2}} - \overline{g_{r,k}}}_\ast^2\big] \\
    &= \mathbb{E}\Big[\norm{ \big[ \overline{g_{r,k+1/2}} - \frac{1}{M}\sum_{m=1}^M g_m(z_{r, k+1/2}^m) \big] \\
    & \qquad + \big[ \frac{1}{M}\sum_{m=1}^M g_m(z_{r, k}^m) -\overline{g_{r,k}} \big] + \frac{1}{M}\sum_{m=1}^M \big[g_m(z_{r, k+1/2}^m) - g_m(\widehat{z_{r, k+1/2}})\big] \\
    & \qquad  + \frac{1}{M}\sum_{m=1}^M \big[g_m(\widehat{z_{r, k}}) - g_m(z_{r, k}^m)\big] + \frac{1}{M}\sum_{m=1}^M \big[g_m(\widehat{z_{r, k+1/2}}) - g_m(\widehat{z_{r, k}})\big] }_\ast^2 \Big] \\
     &\leq 5 \underbrace{\mathbb{E}\Big[ \nrm{\overline{g_{r,k+1/2}} - \frac{1}{M}\sum_{m=1}^M g_m(z_{r, k+1/2}^m)}_\ast^2 \Big]}_{C1} + 5 \underbrace{\mathbb{E}\Big[ \nrm{\frac{1}{M}\sum_{m=1}^M g_m(z_{r, k}^m) -\overline{g_{r,k}}}_\ast^2 \Big]}_{C2} \\
    & \quad + 5 \underbrace{\mathbb{E}\Big[ \nrm{\frac{1}{M}\sum_{m=1}^M \big[g_m(z_{r, k+1/2}^m) - g_m(\widehat{z_{r, k+1/2}})\big]}_\ast^2 \Big]}_{C3} \\
     & \quad + 5 \underbrace{\mathbb{E}\Big[ \nrm{\frac{1}{M}\sum_{m=1}^M \big[g_m(\widehat{z_{r, k}}) - g_m(z_{r, k}^m)\big]}_\ast^2 \Big]}_{C4} + 5 \underbrace{\mathbb{E}\Big[ \nrm{\frac{1}{M}\sum_{m=1}^M \big[g_m(\widehat{z_{r, k+1/2}}) - g_m(\widehat{z_{r, k}})\big]}_\ast^2 \Big]}_{C5}
\end{align*}

For $C1$, by Assumption \ref{asm:grad-op}b and its following Remark \ref{rmk:martingale},
\begin{align*}
    C1 &= \mathbb{E}_{\mc{F}_{r,k+1/2}}\Big[ \nrm{\frac{1}{M}\sum_{m=1}^M g_m(z^m_{r,k+1/2}; \xi^m_{r,k+1/2}) - \frac{1}{M}\sum_{m=1}^M g_m(z_{r, k+1/2}^m)}_\ast^2 \Big] \\
    &=  \frac{1}{M^2} \mathbb{E}_{\mc{F}_{r,k+1/2}} \Big[  \nrm{\sum_{m=1}^M\big[ g_m(z^m_{r,k+1/2}; \xi^m_{r,k+1/2}) - g_m(z_{r, k+1/2}^m)\big]}_\ast^2 \Big] \\
    &= \frac{1}{M^2}\Var_{\mc{F}_{r,k+1/2}} \Big[  \sum_{m=1}^M \big[ g_m(z^m_{r,k+1/2}; \xi^m_{r,k+1/2}) - g_m(z_{r, k+1/2}^m)\big] \Big] \\
    &= \frac{1}{M^2} \sum_{m=1}^M \Var_{\mc{F}_{r,k+1/2}} \Big[   \big[ g_m(z^m_{r,k+1/2}; \xi^m_{r,k+1/2}) - g_m(z_{r, k+1/2}^m)\big] \Big] & \text{(Clients are i.i.d.)}\\
    &= \frac{1}{M^2} \sum_{m=1}^M \mathbb{E}_{\mc{F}_{r,k+1/2}} \Big[ \nrm{ g_m(z^m_{r,k+1/2}; \xi^m_{r,k+1/2}) - g_m(z_{r, k+1/2}^m)}_\ast^2 \Big] \\
    &= \frac{1}{M^2} \sum_{m=1}^M \mathbb{E}_{\mc{F}_{r,k}} \Big[ \mathbb{E}_{\mc{F}_{r,k+1/2}}\big[ \nrm{ g_m(z^m_{r,k+1/2}; \xi^m_{r,k+1/2}) - g_m(z_{r, k+1/2}^m)}_\ast^2 \big\vert \mc{F}_{r,k} \big] \Big] \\
    &\leq \frac{\sigma^2}{M} 
\end{align*}
Similarly, we have $C2 \leq \frac{\sigma^2}{M}$.

For $C3$, by Lemma \ref{lem:l1-l2}, $\beta$-smoothness of $f_m$, and finally Lemma \ref{lem:client-server-z}, we have
\begin{align*}
    C3 &\leq \mathbb{E}\Big[ \frac{1}{M^2} \cdot M \sum_{m=1}^M \nrm{ g_m(z_{r, k+1/2}^m) - g_m(\widehat{z_{r, k+1/2}})}_\ast^2 \Big] \\
    &\leq \frac{\beta^2}{M} \sum_{m=1}^M \mathbb{E}\Big[\nrm{ z_{r, k+1/2}^m - \widehat{z_{r, k+1/2}}}^2 \Big] \\
    &\leq 4\beta^2(\eta^c)^2(k+1)^2G^2
\end{align*}
Similarly for $C4$, we have $C4 \leq 4\beta^2(\eta^c)^2k^2G^2$.

For $C5$, by Lemma \ref{lem:l1-l2}, $\beta$-smoothness of $f_m$ from Assumption \ref{asm:grad-op}a, and finally Lemma \ref{lem:client-server-z},
\begin{align*}
    C5 &= \mathbb{E}\Big[ \frac{1}{M^2}\nrm{\sum_{m=1}^M \big[g_m(\widehat{z_{r, k+1/2}}) - g_m(\widehat{z_{r, k}})\big]}_\ast^2 \Big] \\
    &\leq \mathbb{E}\Big[ \frac{1}{M^2} \cdot M \sum_{m=1}^M \nrm{ g_m(\widehat{z_{r, k+1/2}})) - g_m(\widehat{z_{r, k}})}_\ast^2 \Big] \\
    &\leq \beta^2 \mathbb{E}\Big[\nrm{ \widehat{z_{r, k+1/2}} - \widehat{z_{r, k}}}^2 \Big].
\end{align*}
Plugging the bounds for $C1, \ C2, C3, C4$, and $C5$ back in completes the proof.
\end{proof}

\section{Complete Analysis of FeDualEx for Composite Convex Optimization} \label{appx:FeDualEx-Convex}
In this section, we reduce the problem to composite convex optimization in the following form:
\begin{align}
  \min_{x\in\mc{X}} \phi(x) = f(x) + \psi(x)
\end{align}
where $f(x) = \frac{1}{M}\sum_{m=1}^Mf_m(x)$.
The analysis builds upon the strong-convexity of the distance-generating function $h$ in Assumption \ref{asm:strong-convex} and the following set of assumptions in the convex optimization setting:
\begin{assumption} \label{asm:convex-asms} We make the following assumptions:
    \begin{itemize}
        \item [a.] (Convexity of $f$) $\forall m \in [M]$, $f_m$ is convex. That is, $\forall x, x' \in \mc{X}$,
        \begin{align*}
            f_m(x) - f_m(x') \leq \langle f_m(x) , x - x' \rangle.
        \end{align*}
        \item [b.] (Local Smoothness of $f$) $\forall m \in [M]$, $f_m$ is $\beta$-smooth: $\forall x, x' \in \mc{X}$,
        \begin{align*}
            f_m(x) \leq f_m(x') + \langle f_m(x'), x - x' \rangle + \frac{\beta}{2} \nrm{x - x'}.
        \end{align*}
        \item [c.] (Convexity of $\psi$) $\psi(x)$ is convex.
        \item [d.] (Local Unbiased Estimate and Bounded Variance) For any client $m \in [M]$, the local gradient queried by some local random sample $\xi^m$ is unbiased and also bounded in variance, i.e., $\mathbb{E}_\xi [g_m(x^m;\xi^m)] = g_m(x^m)$ and $\mathbb{E}_\xi [\nrm{g_m(x_m;\xi_m) - g_m(x_m)}_\ast^2] \leq \sigma^2$.
        \item [e.] (Bounded Gradient) $\forall m \in [M]$, $\nrm{g_m(x_m;\xi_m)}_\ast \leq G$.
    \end{itemize}
\end{assumption}
Federated dual extrapolation for composite convex optimization is to replace the part of Algorithm \ref{alg:fed-DualEx} highlighted in {\colorbox{mygreen!30}{green}} with the following updating sequence, where we overuse $\varsigma$ now as the notation for dual variables in the convex setting as well.
\begin{align*}
    &\varsigma^m_{r,0} = \varsigma_{r} \\
    &\text{\textbf{for}} \ k = 0,1,\dots,K-1 \ \text{\textbf{do}} \\
    &\qquad x^m_{r,k} = \tilde{\prox{}}_{\bar{\varsigma}}^{h_{r,k}} (\varsigma_{r,k}^m)\\
    &\qquad x^m_{r,k+1/2} = \tilde{\prox{}}_{\bar{\varsigma} - \varsigma_{r,k}^m}^{h_{r,k+1}} (\eta^c g_m(x_{r,k}^m; \xi^m_{r,k}))\\
    &\qquad \varsigma^m_{r,k+1} = \varsigma^m_{r,k} + \eta^c g_m(x^m_{r,k+1/2}; \xi^m_{r,k+1/2}) \\
    &\text{\textbf{end for}}
\end{align*}
For the proximal operator defined by $h_{r,k}$, reformulating from its Definition \ref{def:generalized-prox} to $\nabla h_{r,k}^\ast$ in Definition \ref{def:conjugate-h} yields
{\begin{align*}
    x^m_{r,k} &= \argmin_x\{\langle \varsigma_{r,k}^m - \bar{\varsigma}, x\rangle  + h_{r,k}(x)\} = \nabla h_{r,k}^\ast(\bar{\varsigma} - \varsigma_{r,k}^m)\\
    x^m_{r,k+1/2} &= \argmin_x\{\langle \eta^c g_m(x_{r,k}^m; \xi^m_{r,k}) - (\bar{\varsigma} - \varsigma_{r,k}^m), x \rangle + h_{r,k+1} (x)\} = \nabla h_{r,k+1}^\ast((\bar{\varsigma} - \varsigma_{r,k}^m) - \eta^c g_m(x_{r,k}^m; \xi^m_{r,k})) \\
    \varsigma^m_{r,k+1} &= \varsigma^m_{r,k} + \eta^c g_m(x^m_{r,k+1/2}; \xi^m_{r,k+1/2})
\end{align*}}
Similarly, we define auxiliary dual variable $\mu^m_{r,k} = \bar{\varsigma} - \varsigma_{r,k}^m$ and $\mu^m_{r,k+1/2}$ the dual image of $x^m_{r,k+1/2}$. Then by definition, $x^m_{r,k} = \nabla h_{r,k}^\ast(\mu^m_{r,k})$ and $x^m_{r,k+1/2} = \nabla h_{r,k+1}^\ast(\mu^m_{r,k+1/2})$. The updating sequence is equivalent to 
\begin{align*}
    \mu^m_{r,k+1/2} &= {\mu^m_{r,k}} - \eta g_m(x^m_{r,k}; \xi^m_{r,k}) \\
    \mu^m_{r,k+1} &= \mu^m_{r,k} - \eta g_m(x^m_{r,k+1/2}; \xi^m_{r,k+1/2}).
\end{align*}
For the shadow sequence of averaged variables $\overline{\mu_{r,k}} = \frac{1}{M} \sum_{m=1}^M \mu^m_{r,k}$ and $\overline{g_{r,k}} = \frac{1}{M}\sum_{m=1}^M g_m(x^m_{r,k}; \xi^m_{r,k})$,
\begin{align}
    \overline{\mu_{r,k+1/2}} &= \overline{\mu_{r,k}} - \eta^c \overline{g_{r,k}}, \label{eq:seq3}\\
    \overline{\mu_{r,k+1}} &= \overline{\mu_{r,k}} - \eta^c \overline{g_{r,k+1/2}}. \label{eq:seq4}
\end{align}
Finally, the projections of the averaged dual back to the primal space are $\widehat{x_{r, k}} = \nabla h_{r,k}^\ast (\overline{\mu_{r,k}})$ and $\widehat{x_{r, k+1/2}} = \nabla h_{r,k+1}^\ast (\overline{\mu_{r,k+1/2}})$

\renewcommand{\thetheorem}{\ref{thm:convex}}
\begin{theorem} Under Assumption \ref{asm:convex-asms}, the ergodic intermediate sequence generated by FeDualEx for composite convex objectives satisfies
\begin{align*}
    \mathbb{E}\big[\phi(\frac{1}{RK}\sum_{r=0}^{R-1}\sum_{k=0}^{K-1}\widehat{x_{r, k+1/2}}) - \phi(x) \big] &\leq \frac{B}{\eta^cRK} + 20\beta^2(\eta^c)^3K^2G^2 + \frac{5\sigma^2\eta^c}{M} + 2\beta(\eta^c)^3K^2G^2.
\end{align*}
Choosing step size
\begin{align*}
    \eta^c = \min \{\frac{1}{5\beta^2}, \frac{B^\frac{1}{4}}{20^\frac{1}{4}\beta^\frac{1}{2}G^\frac{1}{2}K^\frac{3}{4}R^\frac{1}{4}}, \frac{B^\frac{1}{2}M^\frac{1}{2}}{5^\frac{1}{2}\sigma R^\frac{1}{2} K^\frac{1}{2}}, \frac{B^\frac{1}{3}}{2^\frac{1}{3}\beta^\frac{1}{3}G^\frac{2}{3}KR^\frac{1}{3}}\}
\end{align*}
further yields the following convergence rate:
\begin{align*}
    \mathbb{E}\big[\phi(\frac{1}{RK}\sum_{r=0}^{R-1}\sum_{k=0}^{K-1}\widehat{x_{r, k+1/2}}) - \phi(x) \big] &\leq \frac{5\beta^2B}{RK} + \frac{20^\frac{1}{4}\beta^\frac{1}{2}G^\frac{1}{2}B^\frac{3}{4}}{K^\frac{1}{4}R^\frac{3}{4}} + \frac{5^\frac{1}{2}\sigma B^\frac{1}{2}}{M^\frac{1}{2}R^\frac{1}{2} K^\frac{1}{2}} + \frac{2^\frac{1}{3}\beta^\frac{1}{3}G^\frac{2}{3}B^\frac{2}{3}}{R^\frac{2}{3}}.
\end{align*}
\end{theorem}
\begin{proof}
    As the proof for Theorem \ref{thm:main}, the proof for this theorem depends on Lemma \ref{lem:non-smooth-convex} and Lemma \ref{lem:smooth-convex}, which further yield Lemma \ref{lem:main-lem-convex}. These lemmas are presented and proved right after this theorem. Here, we start from Lemma \ref{lem:main-lem-convex}. Telescoping over all $k \in \{0, ..., K-1\}$ and all $r \in \{0, ..., R-1\}$ assuming $\eta^s = 1$ yields
\begin{align*}
    \eta^c\mathbb{E}\big[\sum_{r=0}^{R-1}\sum_{k=0}^{K-1}\phi(\widehat{x_{r, k+1/2}}) - RK\phi(x) \big] &\leq \tilde{V}_{\overline{\mu_{0,0}}}^{h_{0,0}}(x) -\tilde{V}_{\overline{\mu_{R,K}}}^{h_{R,K}}(x) + \frac{5\sigma^2(\eta^c)^2KR}{M} \\
    &\quad + 20\beta^2(\eta^c)^4K^3RG^2 + 2\beta(\eta^c)^3K^3RG^2.
\end{align*}
By Assumption \ref{asm:compactness}, $\tilde{V}_{\overline{\mu_{0,0}}}^{h_{0,0}}(x) = V_{x_0}^{h}(x) \leq B$, where $x_0 = \nabla h^\ast(\bar{\varsigma})$. Dividing both sides by $\eta^cKR$ followed by applying Jensen's inequality (Lemma \ref{lem:jensen}) completes the proof. 
\end{proof}

\begin{lemma} [Bounding the Regularization Term] \label{lem:non-smooth-convex} $\forall x$,
\begin{align*}
 \eta^c \big[ \psi(\widehat{x_{r, k+1/2}}) - \psi(x)\big] &= \tilde{V}_{\overline{\mu_{r,k}}}^{h_{r,k}}(x)  - \tilde{V}_{\overline{\mu_{r,k+1}}}^{h_{r,k+1}}(x) - \tilde{V}_{\overline{\mu_{r,k}}}^{h_{r,k}}(\widehat{x_{r, k+1/2}}) - \tilde{V}_{\overline{\mu_{r,k+1/2}}}^{h_{r,k+1}}(\widehat{x_{r, k+1}}) \\ 
    & \quad + \eta^c\langle \overline{g_{r,k+1/2}} - \overline{g_{r,k}}, \widehat{x_{r, k+1/2}} - \widehat{x_{r, k+1}} \rangle + \eta^c\langle \overline{g_{r,k+1/2}}, x - \widehat{x_{r, k+1/2}} \rangle
\end{align*}
\end{lemma}
\begin{proof}
    The proof of this Lemma is almost identical to the proof of Lemma \ref{lem:non-smooth} with a mere change of variables and distance-generating function from saddle point setting to convex setting.
\end{proof}

The following Lemma highlights the primary difference in the analysis of convex optimization and saddle point optimization. The smoothness of $f_m$ provides an alternative presentation to gradient Lipschitzness that establishes the connection between $\widehat{x_{r, k+1/2}}$, the primal projection of averaged dual on the central server, and $x_{r,k+1/2}^m$ on each client.

\begin{lemma}[Bounding the Smooth Term] \label{lem:smooth-convex}
$\forall x$, 
\begin{align*}
    f(\widehat{x_{r, k+1/2}}) - f(x) &\leq \langle \overline{g_{r,k+1/2}}, \widehat{x_{r, k+1/2}} - x\rangle + \langle \frac{1}{M} \sum_{m=1}^M g_m(x_{r,k+1/2}^m) - \overline{g_{r,k+1/2}}, \widehat{x_{r, k+1/2}} - x\rangle \\ 
    &\quad + \frac{\beta}{2M}\sum_{m=1}^M\nrm{\widehat{x_{r, k+1/2}} - x_{r,k+1/2}^m}^2.
\end{align*}
\end{lemma}
\begin{proof}
    By the smoothness $f_m$ in the form of Assumption \ref{asm:convex-asms}b and then the convexity of $f_m$ in the form of Assumption \ref{asm:convex-asms}a,
    \begin{align*}
        f_m(\widehat{x_{r, k+1/2}}) &\leq f_m(x_{r,k+1/2}^m) + \langle g_m(x_{r,k+1/2}^m), \widehat{x_{r, k+1/2}} - x_{r,k+1/2}^m \rangle + \frac{\beta}{2} \nrm{\widehat{x_{r,k+1/2}} - x_{r,k+1/2}^m}^2  \\
        &\leq f_m(x_{r,k+1/2}^m) + \langle g_m(x_{r,k+1/2}^m), \widehat{x_{r, k+1/2}} - x_{r,k+1/2}^m \rangle + \frac{\beta}{2} \nrm{\widehat{x_{r,k+1/2}} - x_{r,k+1/2}^m}^2 \\
        &\quad + f_m(x) - f_m(x_{r,k+1/2}^m) + \langle g_m(x_{r,k+1/2}^m),  x_{r,k+1/2}^m - x \rangle \\
        &\leq f_m(x) + \langle g_m(x_{r,k+1/2}^m), \widehat{x_{r, k+1/2}} - x \rangle + \frac{\beta}{2} \nrm{\widehat{x_{r,k+1/2}} - x_{r,k+1/2}^m}^2 
    \end{align*}
Then for function $f = \frac{1}{M} \sum_{m=1}^M f_m$,
\begin{align*}
    f(\widehat{x_{r, k+1/2}}) - f(x) &\leq \frac{1}{M}\sum_{m=1}^M \big[ f_m(\widehat{x_{r, k+1/2}}) - f_m(x) \big] \\
    &\leq \langle \frac{1}{M}\sum_{m=1}^M g_m(x_{r,k+1/2}^m), \widehat{x_{r, k+1/2}} - x \rangle + \frac{1}{M}\sum_{m=1}^M \frac{\beta}{2} \nrm{\widehat{x_{r,k+1/2}} - x_{r,k+1/2}^m}^2 \\
    &=\langle \overline{g_{r,k+1/2}}, \widehat{x_{r, k+1/2}} - x \rangle + \langle \frac{1}{M}\sum_{m=1}^M g_m(x_{r,k+1/2}^m) - \overline{g_{r,k+1/2}}, \widehat{x_{r, k+1/2}} - x \rangle \\
    &\quad + \frac{\beta}{2M}\sum_{m=1}^M \nrm{\widehat{x_{r,k+1/2}} - x_{r,k+1/2}^m}^2.
\end{align*}
\end{proof}

Now we are ready to present the main lemma that combines Lemma \ref{lem:non-smooth-convex} and Lemma \ref{lem:smooth-convex}. For the proof, we utilize again Lemma \ref{lem:unbiased-grad}, Lemma \ref{lem:client-server-z}, and Lemma \ref{lem:gradient-diff}, all of which we claim to hold trivially in the composite convex optimization setting.
\begin{lemma}[Main Lemma for FeDualEx in Composite Convex Optimization] \label{lem:main-lem-convex} Under Assumption \ref{asm:convex-asms},
\begin{align*}
    \eta^c\mathbb{E}\big[\phi(\widehat{x_{r, k+1/2}}) - \phi(x) \big] &\leq \tilde{V}_{\overline{\mu_{r,k}}}^{h_{r,k}}(x) -\tilde{V}_{\overline{\mu_{r,k+1}}}^{h_{r,k+1}}(x) + \frac{5\sigma^2\eta^c}{M} + 10\beta^2(\eta^c)^3(2k^2+2k+1)G^2 \\
    &\quad + \frac{(\eta^c)^2\sigma^2}{2M(1-\eta^c)} +  2 \beta (\eta^c)^3 (k+1)^2 G^2.
\end{align*}
\end{lemma}
\begin{proof}
Summing the results in Lemma \ref{lem:non-smooth-convex} and Lemma \ref{lem:smooth-convex}:
{\begin{align*}
    \eta^c\big(\phi(\widehat{x_{r, k+1/2}}) - \phi(x) \big) &\leq \tilde{V}_{\overline{\mu_{r,k}}}^{h_{r,k}}(x) - \tilde{V}_{\overline{\mu_{r,k+1}}}^{h_{r,k+1}}(x) - \tilde{V}_{\overline{\mu_{r,k}}}^{h_{r,k}}(\widehat{x_{r, k+1/2}}) - \tilde{V}_{\overline{\mu_{r,k+1/2}}}^{h_{r,k+1}}(\widehat{x_{r, k+1}})  \\
    & \quad + \eta^c\langle \overline{g_{r,k+1/2}} - \overline{g_{r,k}}, \widehat{x_{r, k+1/2}} - \widehat{x_{r, k+1}} \rangle + \frac{\eta^c\beta}{2M}\sum_{m=1}^M\nrm{\widehat{x_{r, k+1/2}} - x_{r,k+1/2}^m}^2 \\
    &\quad + \eta^c\langle \frac{1}{M} \sum_{m=1}^M g_m(x_{r,k+1/2}^m) - \overline{g_{r,k+1/2}}, \widehat{x_{r, k+1/2}} - x\rangle.
\end{align*}}

For the latter two generalized Bregman divergence terms $- \tilde{V}_{\overline{\mu_{r,k}}}^{h_{r,k}}(\widehat{x_{r, k+1/2}}) - \tilde{V}_{\overline{\mu_{r,k+1/2}}}^{h_{r,k+1}}(\widehat{x_{r, k+1}})$, we bound them by Lemma \ref{lem:bregman-larger} and the strong convexity of $h$ in Assumption \ref{asm:strong-convex}. As a result,
\begin{align*}
    \eta^c\big(\phi(\widehat{x_{r, k+1/2}}) - \phi(x) \big) &\leq \tilde{V}_{\overline{\mu_{r,k}}}^{h_{r,k}}(x) - \tilde{V}_{\overline{\mu_{r,k+1}}}^{h_{r,k+1}}(x) - \frac{1}{2} \nrm{\widehat{x_{r,k}} - \widehat{x_{r, k+1/2}}}^2   \\
    & \quad \underbrace{ - \frac{1}{2} \nrm{\widehat{x_{r,k+1/2}} - \widehat{x_{r, k+1}}}^2 + \eta^c\langle \overline{g_{r,k+1/2}} - \overline{g_{r,k}}, \widehat{x_{r, k+1/2}} - \widehat{x_{r, k+1}} \rangle}_{A} \\
    &\quad + \langle \frac{\eta^c}{M} \sum_{m=1}^M g_m(x_{r,k+1/2}^m) - \overline{g_{r,k+1/2}}, \widehat{x_{r, k+1/2}} - x\rangle   + \frac{\eta^c\beta}{2M}\sum_{m=1}^M\nrm{\widehat{x_{r, k+1/2}} - x_{r,k+1/2}^m}^2.
\end{align*}
$A$ can be bounded with Cauchy-Schwarz inequality (Lemma \ref{lem:cauchy-schwarz}) and Young's inequality (Lemma \ref{lem:young}).
\begin{align*}
    A &\leq - \frac{1}{2}\nrm{\widehat{x_{r,k+1/2}} - \widehat{x_{r,k+1}}}^2 + \eta^c\nrm{\overline{g_{r,k+1/2}} - \overline{g_{r,k}}}_\ast\nrm{\widehat{x_{r, k+1/2}} - \widehat{x_{r, k+1}}} \\
    &\leq - \frac{1}{2}\nrm{\widehat{x_{r,k+1/2}} - \widehat{x_{r,k+1}}}^2 + \frac{(\eta^c)^2}{2}\nrm{\overline{g_{r,k+1/2}} - \overline{g_{r,k}}}_\ast^2 + \frac{1}{2}\nrm{\widehat{x_{r, k+1/2}} - \widehat{x_{r, k+1}}}^2 \\
    &= \frac{(\eta^c)^2}{2}\nrm{\overline{g_{r,k+1/2}} - \overline{g_{r,k}}}_\ast^2.
\end{align*}
Taking expectations on both sides we get
\begin{align*}
    \eta^c\mathbb{E}\big[\phi(\widehat{x_{r, k+1/2}}) - \phi(x) \big] &\leq \tilde{V}_{\overline{\mu_{r,k}}}^{h_{r,k}}(x) -\tilde{V}_{\overline{\mu_{r,k+1}}}^{h_{r,k+1}}(x) \underbrace{- \frac{1}{2} \mathbb{E}\big[\nrm{\widehat{x_{r,k}} - \widehat{x_{r, k+1/2}}}^2\big]}_{B1} + \underbrace{\frac{(\eta^c)^2}{2}\mathbb{E}\big[\nrm{\overline{g_{r,k+1/2}} - \overline{g_{r,k}}}_\ast^2 \big] }_{B2} \\
    &\qquad + \underbrace{ \mathbb{E} \big[ \langle \frac{\eta^c}{M} \sum_{m=1}^M g_m(x_{r,k+1/2}^m) - \overline{g_{r,k+1/2}}, \widehat{x_{r, k+1/2}} - x\rangle\big]}_{B3} \\
    &\qquad + \underbrace{\frac{\eta^c\beta}{2M}\sum_{m=1}^M\mathbb{E}\big[\nrm{\widehat{x_{r, k+1/2}} - x_{r,k+1/2}^m}^2\big]}_{B4}.
\end{align*}
B2 is bounded in Lemma \ref{lem:gradient-diff}. Therefore, for $\eta^c \leq \frac{1}{5\beta^2}$,
\begin{align*}
    B1 + B2 &\leq \frac{5\sigma^2(\eta^c)^2}{M} + 20\beta^2(\eta^c)^4(k+1)^2G^2.
\end{align*}
B3 is zero after taking the expectation by Lemma \ref{lem:unbiased-grad}. B4 is bounded in Lemma \ref{lem:client-server-z}. Plugging the bounds for $B1+B2, \ B3$, and $B4$ back in completes the proof.
\end{proof}

\section{FeDualEx in Other Settings} \label{appx:FeDualEx-Others}
In this section, we provide the algorithm along with the convergence rate for sequential versions of FeDualEx. The proofs in this section rely only on the Lipschitzness of the gradient operator. As a result, the analysis applies to both composite saddle point optimization and composite convex optimization.

\begin{algorithm}[t]
\renewcommand\thealgorithm{3}
   \caption{\textsc{Stochastic-Dual-Extrapolation} for Composite SPP} 
   \label{alg:stochastic-DualEx}
\begin{algorithmic}
     \REQUIRE $\phi(z) = f(x,y) + \psi_1(x) - \psi_2(y)$: objective function; $\ell(z)$: distance-generating function; $g(z) = (\nabla_x f (x,y), -\nabla_y f (x,y))$: gradient operator.
    \renewcommand{\algorithmicrequire}{\textbf{Hyperparameters:}} \REQUIRE $T$: number of iterations; $\eta$: step size.
    \renewcommand{\algorithmicrequire}{\textbf{Dual Initialization:}} \REQUIRE $\varsigma_{0} = 0$: initial dual variable, $\bar{\varsigma} \in \mc{S}$: fixed point in the dual space.
    \ENSURE Approximate solution $z = (x, y)$ to $ \min_{x \in \mc{X}} \max_{y \in \mc{Y}} \phi(x, y)$ 
        \FOR{$t = 0,1,\dots,T-1$}
            \STATE $z_t = \tilde{\prox{}}_{\bar{\varsigma}}^{\ell_{t}} (\varsigma_t)$ \hfill $\vartriangleright$ Two-step evaluation of the generalized proximal operator
            \STATE $z_{t+1/2} = \tilde{\prox{}}_{\bar{\varsigma} - \varsigma_{t}}^{\ell_{t}} (\eta^c g(z_{t}; \xi_{t}))$
            \STATE $\varsigma_{t+1} = \varsigma_{t} + \eta^c g(z_{t+1/2}; \xi_{t+1/2})$ \hfill $\vartriangleright$ Dual variable update
       \ENDFOR
       \STATE \textbf{end for}
   \STATE \textbf{Return:} $\frac{1}{T} \sum_{t=0}^{T-1}z_{t+1/2}$.
\end{algorithmic}
\end{algorithm}

\subsection{Stochastic Dual Extrapolation for Composite Saddle Point Optimization} \label{appx:stochastic}
The sequential version of FeDualEx immediately yields Algorithm \ref{alg:stochastic-DualEx}, stochastic dual extrapolation for Composite SPP. This algorithm generalizes dual extrapolation to both  composite and smooth stochastic saddle point optimization with the latter taking $\psi(z) = 0$. Its convergence rate is analyzed in the following theorem, which to the best of our knowledge, is the first one for stochastic composite saddle point optimization.
\renewcommand{\thetheorem}{\ref{thm:stochastic}}
\begin{theorem} Under the sequential version of Assumption \ref{asm:obj-func}-\ref{asm:compactness}, namely with $M = 1$, $\forall z \in \mc{Z}$, the ergodic intermediate sequence generated by Algorithm \ref{alg:stochastic-DualEx} satisfies
\begin{align*} 
    \mathbb{E}\big[\phi(\frac{1}{T}\sum_{t=0}^{T-1}z_{t+1/2}) - \phi(z) \big] &\leq \frac{B}{\eta T} + 3\sigma^2\eta.
\end{align*}
Choosing step size
\begin{align*}
    \eta = \min \{\frac{1}{3\beta^2}, \frac{B^\frac{1}{2}}{3^\frac{1}{2}\sigma T^\frac{1}{2}}\},
\end{align*}
further yields the following convergence rate:
\begin{align*}
    \mathbb{E}\big[\phi(\frac{1}{T}\sum_{t=0}^{T-1}z_{t+1/2}) - \phi(z) \big] &\leq \frac{3\beta^2B}{T} + \frac{3^\frac{1}{2}\sigma B^\frac{1}{2}}{T^\frac{1}{2}}.
\end{align*}
\end{theorem}
\begin{proof}
By proof similar to Lemma \ref{lem:non-smooth}, we have 
\begin{align*}
 \eta \big[ \psi(z_{t+1/2}) - \psi(z)\big] &= \tilde{V}_{\omega_{t}}^{\ell_{t}}(z)  - \tilde{V}_{\omega_{t+1}}^{\ell_{t+1}}(z) - \tilde{V}_{\omega_{t}}^{\ell_{t}}(z_{t+1/2}) - \tilde{V}_{\omega_{t+1/2}}^{\ell_{t+1}}(z_{t+1}) \\ 
    & \quad + \eta\langle g_{t+1/2} - g_{t}, z_{t+1/2} - z_{t+1} \rangle + \eta\langle g_{t+1/2}, z - z_{t+1/2} \rangle \\
    &\leq \tilde{V}_{\omega_{t}}^{\ell_{t}}(z)  - \tilde{V}_{\omega_{t+1}}^{\ell_{t+1}}(z) \\
    & \quad \underbrace{- \frac{1}{2} \nrm{z_{t} - z_{t+1/2}}^2 - \frac{1}{2} \nrm{z_{t+1/2} - z_{t+1}}^2 + \eta\langle g_{t+1/2} - g_{t}, z_{t+1/2} - z_{t+1} \rangle}_{A} \\
    & \quad + \underbrace{\eta \langle g(z_{t+1/2}) - g_{t+1/2}, z_{t+1/2} - z\rangle}_{B} - \eta\langle g(z_{t+1/2}), z_{t+1/2} - z \rangle.
\end{align*}
where the inequality holds by Lemma \ref{lem:bregman-larger} and the strong convexity of $\ell$ in Remark \ref{rmk:extension}, and then simply expanding the last term to build a connection between the stochastic gradient and true gradient. By Cauchy-Schwarz inequality (Lemma \ref{lem:cauchy-schwarz}), Young's inequality (Lemma \ref{lem:young}), and Lemma \ref{lem:l1-l2},
\begin{align*}
    A &\leq - \frac{1}{2} \nrm{z_{t} - z_{t+1/2}}^2 - \frac{1}{2} \nrm{z_{t+1/2} - z_{t+1}}^2 + \frac{\eta^2}{2}\nrm{g_{t+1/2} - g_{t}}_\ast^2 + \frac{1}{2}\nrm{z_{t+1/2} - z_{t+1}}^2 \\
    &= - \frac{1}{2} \nrm{z_{t} - z_{t+1/2}}^2 + \frac{\eta^2}{2}\nrm{[g_{t+1/2} - g(z_{t+1/2})] + [g(z_{t}) - g_{t}] + [g(z_{t+1/2}) - g(z_{t})]}_\ast^2 \\
    &\leq - \frac{1}{2} \nrm{z_{t} - z_{t+1/2}}^2  + \frac{3\eta^2}{2}\nrm{g(z_{t+1/2}) - g(z_{t})}_\ast^2 + \frac{3\eta^2}{2}\nrm{g_{t+1/2} - g(z_{t+1/2})}_\ast^2 + \frac{3\eta^2}{2}\nrm{g(z_{t}) - g_{t}}_\ast^2 \\
    &\leq \frac{3\eta^2\beta^2 - 1}{2} \nrm{z_{t} - z_{t+1/2}}^2 + \frac{3\eta^2}{2}\nrm{g_{t+1/2} - g(z_{t+1/2})}_\ast^2 + \frac{3\eta^2}{2}\nrm{g(z_{t}) - g_{t}}_\ast^2,
\end{align*}
where the last inequality holds by the $\beta$-Lipschitzness of the gradient operator. After taking expectations, the last two terms are bounded by the variance of the gradient $\sigma^2$, and $B$ becomes zero by proof similar to Lemma \ref{lem:unbiased-grad}. Therefore, for $\eta \leq \frac{1}{3\beta^2}$
\begin{align*}
    \eta \mathbb{E}\big[\langle g(z_{t+1/2}), z_{t+1/2} - z \rangle + \psi(z_{t+1/2}) - \psi(z)\big] &\leq \tilde{V}_{\omega_{t}}^{\ell_{t}}(z)  - \tilde{V}_{\omega_{t+1}}^{\ell_{t+1}}(z) + 3\eta^2 \sigma^2.
\end{align*}
Telescoping over all $t \in \{0, ..., T-1\}$ and dividing both sides by $\eta T$ completes the proof.
\end{proof}

\subsection{Deterministic Dual Extrapolation for Composite Saddle Point Optimization} \label{appx:deterministic}

\begin{algorithm}[t]
\renewcommand\thealgorithm{4}
   \caption{\textsc{Composite-Dual-Extrapolation}} 
   \label{alg:composite-DualEx}
\begin{algorithmic}
     \REQUIRE $\phi(z) = f(x,y) + \psi_1(x) - \psi_2(y)$: objective function; $\ell(z)$: distance-generating function; $g(z) = (\nabla_x f (x,y), -\nabla_y f (x,y))$: gradient operator.
    \renewcommand{\algorithmicrequire}{\textbf{Hyperparameters:}} \REQUIRE $T$: number of iterations; $\eta$: step size.
    \renewcommand{\algorithmicrequire}{\textbf{Dual Initialization:}} \REQUIRE $\varsigma_{0} = 0$: initial dual variable, $\bar{\varsigma} \in \mc{S}$: fixed point in the dual space.
    \ENSURE Approximate solution $z = (x, y)$ to $ \min_{x \in \mc{X}} \max_{y \in \mc{Y}} \phi(x, y)$ 
        \FOR{$t = 0,1,\dots,T-1$}
            \STATE $z_t = \tilde{\prox{}}_{\bar{\varsigma}}^{\ell_{t}} (\varsigma_t)$ \hfill $\vartriangleright$ Two-step evaluation of the generalized proximal operator
            \STATE $z_{t+1/2} = \tilde{\prox{}}_{\bar{\varsigma} - \varsigma_{t}}^{\ell_{t}} (\eta^c g(z_{t}))$
            \STATE $\varsigma_{t+1} = \varsigma_{t} + \eta^c g(z_{t+1/2})$ \hfill $\vartriangleright$ Dual variable update
       \ENDFOR
       \STATE \textbf{end for}
   \STATE \textbf{Return:} $\frac{1}{T} \sum_{t=0}^{T-1}z_{t+1/2}$.
\end{algorithmic}
\end{algorithm}

Further removing the data-dependent noise in the gradient, we present the deterministic sequential version of FeDualEx, which still generalizes Nesterov's dual extrapolation \citep{nesterov2007dual} to composite saddle point optimization. As a result, we term this algorithm composite dual extrapolation, as presented in Algorithm \ref{alg:composite-DualEx}.

We also provide a convergence analysis, which shows that composite dual extrapolation achieves the $\mc{O}(\frac{1}{T})$ convergence rate as its original non-composite smooth version \citep{nesterov2007dual}, as well as composite mirror prox (CoMP) \citep{he2015mirror}. We do so with a very simple proof based on the recently proposed notion of relative Lipschitzness \citep{CohenST21}. We start by introducing the definition of relative Lipschitzness and a relevant lemma.

\begin{definition}[Relative Lipschitzness (Definition 1 in \citet{CohenST21})] \label{def:relative-lipschitz} For convex distance-generating function $h:\mc{Z} \rightarrow \mathbb{R}$, we call operator $g: \mc{Z} \rightarrow \mc{Z}^\ast$ $\lambda$-relatively Lipschitz with respect to $h$ if  $ \ \forall z, w, u \in \mc{Z}$,
\begin{align*}
    \langle g(w) - g(z), w - u \rangle \leq \lambda (V^h_z(w) + V^h_w(u)).
\end{align*}
\end{definition}

\begin{lemma}[Lemma 1 in \citet{CohenST21}] \label{lem:relative-lipschitz} If $g$ is $\beta$-Lipschitz and $h$ is $\alpha$-strongly convex, $g$ is $\frac{\beta}{\alpha}$-relatively Lipschitz with respect to $h$.
\end{lemma}

\thmdeterministic*
\begin{proof}
By proof similar to Lemma \ref{lem:non-smooth}, we have 
\begin{align*}
 \eta \big[ \psi(z_{t+1/2}) - \psi(z)\big] &= \tilde{V}_{\omega_{t}}^{\ell_{t}}(z)  - \tilde{V}_{\omega_{t+1}}^{\ell_{t+1}}(z) - \tilde{V}_{\omega_{t}}^{\ell_{t}}(z_{t+1/2}) - \tilde{V}_{\omega_{t+1/2}}^{\ell_{t+1}}(z_{t+1}) \\ 
    & \quad + \eta\langle g(z_{t+1/2}) - g(z_{t}), z_{t+1/2} - z_{t+1} \rangle + \eta\langle g(z_{t+1/2}), z - z_{t+1/2} \rangle.
\end{align*}
By Lemma \ref{lem:relative-lipschitz}, we know that $g$ is $\beta$-relatively Lipschitz with respect to $\ell$ under the $\beta$-Lipschitzness assumption of $g$ and $1$-strong convexity assumption of $\ell$. Then by Definition \ref{def:relative-lipschitz}, we have
\begin{align*}
 \eta & \big[ \psi(z_{t+1/2}) - \psi(z) + \langle g(z_{t+1/2}), z_{t+1/2} - z \rangle \big] \\
 &\leq \tilde{V}_{\omega_{t}}^{\ell_{t}}(z)  - \tilde{V}_{\omega_{t+1}}^{\ell_{t+1}}(z) - \tilde{V}_{\omega_{t}}^{\ell_{t}}(z_{t+1/2}) - \tilde{V}_{\omega_{t+1/2}}^{\ell_{t+1}}(z_{t+1}) + \eta^c\langle g(z_{t+1/2}) - g(z_{t}), z_{t+1/2} - z_{t+1} \rangle \\
    &\leq \tilde{V}_{\omega_{t}}^{\ell_{t}}(z)  - \tilde{V}_{\omega_{t+1}}^{\ell_{t+1}}(z) - \tilde{V}_{\omega_{t}}^{\ell_{t}}(z_{t+1/2}) - \tilde{V}_{\omega_{t+1/2}}^{\ell_{t+1}}(z_{t+1})  + \eta^c \beta \big[ V_{z_{t}}^{\ell}(z_{t+1/2}) + V_{z_{t+1/2}}^{\ell}(z_{t+1}) \big] \\
    &\leq \tilde{V}_{\omega_{t}}^{\ell_{t}}(z)  - \tilde{V}_{\omega_{t+1}}^{\ell_{t+1}}(z).
\end{align*}
where the last inequality holds for $\eta \leq \frac{1}{\beta}$ by Lemma \ref{lem:bregman-larger}.
Telescoping over all $t \in \{0, ..., T-1\}$ and dividing both sides by $\eta T$ completes the proof.
\end{proof}

\section{Federated Mirror Prox} \label{appx:FedMiP}
We present Federated Mirror Prox (FedMiP) here in Algorithm \ref{alg:fed-MiP} as a baseline. The part highlighted in {\colorbox{mygreen!30}{green}} resembles the mirror prox algorithm introduced in Section \ref{appx:MP-DE}. We use the composite mirror map representation introduced in Section \ref{appx:MD-DA} to avoid confusion, as the composite proximal operator we proposed for FeDualEx is slightly different from that used in composite mirror descent as discussed in Section \ref{sec:FeDualEx-alg}.

\begin{algorithm}[H]
\renewcommand\thealgorithm{2}
   \caption{\textsc{Federated-Mirror-Prox} (FedMiP) for Composite SPP} 
   \label{alg:fed-MiP}
\begin{algorithmic}[1]
     \REQUIRE $\phi(z) = f(x,y) + \psi_1(x) - \psi_2(y) = \frac{1}{M}\sum_{m=1}^M f_m(\cdot) + \psi_1(x) - \psi_2(y)$: objective function; $\ell(z)$: distance-generating function; $g_m(z) = (\nabla_x f_m (x,y), -\nabla_y f_m (x,y))$: gradient operator.
    \renewcommand{\algorithmicrequire}{\textbf{Hyperparameters:}} \REQUIRE $R$: number of rounds of communication; $K$: number of local update iterations; $\eta^s$: server step size; $\eta^c$: client step size.
    \renewcommand{\algorithmicrequire}{\textbf{Primal Initialization:}} \REQUIRE $z_{0}$: initial primal variable.
    \ENSURE Approximate solution $z = (x, y)$ to $ \min_{x \in \mc{X}} \max_{y \in \mc{Y}} \phi(x, y)$ 
   \FOR{$r = 0,1,\dots,R-1$}
       \STATE Sample a subset of clients $C_r \subseteq [M]$
       \FOR{$m \in C_r$ \textbf{in parallel}}
           \STATE $z^m_{r,0} = z_{r}$
           \tikzmk{A}\FOR{$k = 0,1,\dots,K-1$}
               \STATE $z^m_{r,k+1/2} = \nabla(\ell+\eta^c \psi)^\ast(\nabla h(z^m_{r,k}) - \eta^c g(z^m_{r,k}; \xi^m_{r,k}))$
               \STATE $z^m_{r,k+1} = \nabla(\ell+\eta^c \psi)^\ast(\nabla h(z^m_{r,k}) - \eta^c g(z^m_{r,k+1/2}; \xi^m_{r,k+1/2}))$
           \ENDFOR
           \STATE \textbf{end for}
       \ENDFOR
       \STATE \textbf{end parallel for}
       \tikzmk{B} \boxit{mygreen}
       \STATE $\Delta_r = \frac{1}{|\mc{C}_r|}\sum_{m\in \mc{C}_r} (z^m_{r,K} - z^m_{r,0})$
       \STATE $z_{r+1} = \nabla(\ell+\eta^s\eta^cK \psi)^\ast(\nabla h(z_r) + \eta^s \Delta_r)$
   \ENDFOR
   \STATE \textbf{end for}
   \STATE \textbf{Return:} $\frac{1}{RK} \sum_{r=0}^{R-1}\sum_{k=0}^{K-1} z_{r,k+1/2}$.
\end{algorithmic}
\end{algorithm}

\end{document}

%% file: macros.tex
\newcommand{\R}{\mathbb{R}}

\DeclarePairedDelimiter{\nrm}{\big \Vert}{\big \Vert}
\DeclarePairedDelimiter{\norm}{\Big \Vert}{\Big \Vert}

\DeclareMathOperator*{\argmin}{arg\,min} 
\DeclareMathOperator*{\argmax}{arg\,max}

\DeclareMathOperator*{\prox}{Prox} 

\newcommand{\mc}[1]{\mathcal{#1}}

\def\ddefloop#1{\ifx\ddefloop#1\else\ddef{#1}\expandafter\ddefloop\fi}
\def\ddef#1{\expandafter\def\csname 
bb#1\endcsname{\ensuremath{\mathbb{#1}}}}
\ddefloop ABCDEFGHIJKLMNOPQRSTUVWXYZ\ddefloop
\def\ddefloop#1{\ifx\ddefloop#1\else\ddef{#1}\expandafter\ddefloop\fi}
\def\ddef#1{\expandafter\def\csname 
b#1\endcsname{\ensuremath{\mathbf{#1}}}}
\ddefloop ABCDEFGHIJKLMNOPQRSTUVWXYZ\ddefloop
\def\ddef#1{\expandafter\def\csname 
c#1\endcsname{\ensuremath{\mathcal{#1}}}}
\ddefloop ABCDEFGHIJKLMNOPQRSTUVWXYZ\ddefloop
\def\ddef#1{\expandafter\def\csname 
h#1\endcsname{\ensuremath{\widehat{#1}}}}
\ddefloop ABCDEFGHIJKLMNOPQRSTUVWXYZ\ddefloop
\def\ddef#1{\expandafter\def\csname 
hc#1\endcsname{\ensuremath{\widehat{\mathcal{#1}}}}}
\ddefloop ABCDEFGHIJKLMNOPQRSTUVWXYZ\ddefloop
\def\ddef#1{\expandafter\def\csname 
t#1\endcsname{\ensuremath{\widetilde{#1}}}}
\ddefloop ABCDEFGHIJKLMNOPQRSTUVWXYZ\ddefloop
\def\ddef#1{\expandafter\def\csname 
tc#1\endcsname{\ensuremath{\widetilde{\mathcal{#1}}}}}
\ddefloop ABCDEFGHIJKLMNOPQRSTUVWXYZ\ddefloop

\usepackage{nicefrac}

\newsavebox\CBox

\newcommand{\removed}[1]{}

\usepackage{booktabs}
\usepackage{subcaption}
\DeclareSymbolFont{bbold}{U}{bbold}{m}{n}
\DeclareSymbolFontAlphabet{\mathbbold}{bbold}